%% file: qpi.tex
\documentclass{article}

\usepackage[preprint]{neurips_2023} %

\usepackage[utf8]{inputenc} %
\usepackage[T1]{fontenc}    %
\usepackage{hyperref}       %
\usepackage{url}            %
\usepackage{booktabs}       %
\usepackage{amsfonts}       %
\usepackage{nicefrac}       %
\usepackage{microtype}      %
\usepackage{xcolor}         %

\usepackage{amsthm}
\usepackage{stmaryrd}

\usepackage{amsmath}
\usepackage{cases}
\usepackage[capitalise]{cleveref}

 \newcommand{\B}{\mathcal{B}}

 \newcommand{\dfour}{{\sqrt{d_1+1}}}
 \newcommand{\dfoursq}{{(d_1+1)}}

 \newcommand{\tp}{{TP}\xspace}

 \newcommand{\mmax}{{m_\text{max}}}
 \newcommand{\mpmax}{{m'_\text{max}}}
 
 \newcommand{\psuper}[2]{{\mathrm{p}(#1)}}
 \newcommand{\pshort}[1]{\mathrm{p}({#1})}

 \newcommand{\ordot}{\tilde{\ordo}}
 \newcommand{\ordo}{\mathcal{O}}

 \newcommand{\Sr}{R_{H}}
 \newcommand{\sr}{r_{H}}

 \newcommand{\bI}{\mathbf{I}}
 \newcommand{\bJ}{\mathbf{J}}

\newcommand{\optnormconst}{\beta H}

\newcommand{\zero}{\boldsymbol{\mathbf{0}}}

  \newcommand{\pa}{\parallel}
 
 \newcommand{\miota}{{\thetabound^{-2}}}

\newcommand{\pd}{\text{PD}}

\newcommand{\pdseq}{\text{PD}^H}

\newcommand{\psd}{\text{PSD}}

\newcommand{\mge}{\succeq}
\newcommand{\mle}{\preceq}
\newcommand{\thetabound}{L_2}
\newcommand{\featurebound}{L_1}
\newcommand{\precondbound}{L_3}
\DeclareMathOperator{\Tr}{Tr}
\DeclareMathOperator{\Proj}{Proj}

\newcommand{\phio}{{\phi}}  %

\newcommand{\I}[1]{\mathbbm{1}\left\{#1\right\}}
\newcommand{\event}[1]{{\mathcal{E}_{#1}}}
\newcommand{\traj}[1]{s_{#1}\!\!\shortrightarrow}

\renewcommand{\tp}{{\otimes}}
\DeclareMathOperator{\trace}{Tr}
\DeclareMathOperator{\kernel}{Ker}

\DeclareMathOperator{\image}{Im}
\DeclareMathOperator{\range}{range}
\DeclareMathOperator{\clip}{clip}
\DeclareMathOperator{\ord}{ord}
\DeclareMathOperator{\stage}{stage}

\newcommand{\qpieleanor}{{\textsc{SkippyEleanor}}\xspace}
\newcommand{\skippypolicy}{{\textsc{SkippyPolicy}}\xspace}

 \usepackage[noend]{algpseudocode}
\usepackage{algorithmicx}
\usepackage{varwidth}

\usepackage{bbm}

\algnewcommand{\algorithmicgoto}{\textbf{goto}}%
\algnewcommand{\Goto}[1]{\algorithmicgoto~\ref{#1}}%
\algnewcommand{\Break}{\textbf{break}}%

\algnewcommand{\Initialize}[1]{%
  \State \textbf{Initialize:}
  \Statex \hspace*{\algorithmicindent}\parbox[t]{.8\linewidth}{\raggedright #1}
}
\algnewcommand{\Inputs}[1]{%
  \State \textbf{Inputs:}
  \Statex \hspace*{\algorithmicindent}\parbox[t]{.8\linewidth}{\raggedright #1}
}

\input{preamble}

\makeatletter
 \let\Ginclude@graphics\@org@Ginclude@graphics
\makeatother

\usepackage{times}
\usepackage{newtxmath}

\usepackage{thmtools}
\usepackage{thm-restate}
\usepackage{subcaption}

\usepackage[utf8]{inputenc}

\title{
Online RL in Linearly $q^\pi$-Realizable MDPs Is as Easy as in Linear MDPs If You Learn What to Ignore
}

\author{%
 Gell\'ert Weisz\\
 Google DeepMind, London, UK\\
 University College London, London, UK\\
 \And
 Andr\'as {Gy}\"orgy\\
 Google DeepMind, London, UK
 \And
 {Cs}aba {Sz}epesv\'ari\\
 Google DeepMind, Montreal, Canada\\
 University of Alberta, Edmonton, Canada\\
}

\begin{document}

\maketitle

\begin{abstract}
We consider online reinforcement learning (RL) in episodic Markov decision processes (MDPs) under the  linear $q^\pi$-realizability assumption, where it is assumed that the action-values of all policies can be  expressed as linear functions of state-action features. This class is known to be more general than  linear MDPs, where the transition kernel and the reward function are assumed to be linear functions of the feature vectors. As our first contribution, we show that the difference between the two classes is the presence of states in linearly $q^\pi$-realizable MDPs where for any policy, all the actions have  approximately equal values, and skipping over these states by following an arbitrarily fixed policy in those states transforms the problem to a linear MDP. Based on this observation, we derive a novel (computationally inefficient) learning algorithm for linearly $q^\pi$-realizable MDPs that simultaneously learns what states should be skipped over and runs another learning algorithm on the linear MDP hidden in the problem. The method returns an $\epsilon$-optimal policy after $\polylog(H, d)/\epsilon^2$ interactions with the MDP, where $H$ is the time horizon and $d$ is the dimension of the feature vectors, giving the first polynomial-sample-complexity online RL algorithm for this setting. The results are proved for the misspecified case, where the sample complexity is shown to degrade gracefully with the misspecification error.
\end{abstract}

\section{Introduction}
We consider reinforcement learning where an agent interacts in an online fashion with an environment modeled  as a Markov decision process: The agent, observing a state, takes an action that results in a random next state and reward, the latter of which is to be maximized over time.
To tackle large, possibly infinite state spaces, additional structure needs to be introduced to this problem.
One such structure is a ``feature-map'' that maps state-action pairs to $d$-dimensional vectors (for some positive integer $d$) with the intention that a ``good’’  feature-map extracts  important information from the state-action pairs so that learning with this extra information becomes tractable.
An example is the case of \emph{linear MDPs} \citep{Jin_Yang_Wang_Jordan_2019},
where the assumption is that both the transition and reward functions are linearly factorizable and their left factors are given by the feature-map.
In contrast, value-based approaches, such as \emph{$q^\pi$-realizability} \citep{Du_Kakade_Wang_Yan_2019,LaSzeGe19} aim to model only the action-values with the features.
In this work, we focus on the latter, a strictly more general setting than that of linear MDPs \citep[Proposition 4]{zanette2020learning}.

There are several sample-efficient algorithms
discovering near-optimal policies in linear MDPs under various MDP access models and settings (online access: \cite{Jin_Yang_Wang_Jordan_2019}; batch setting: \cite{jin2021pessimism};
reward-free setting: \cite{wagenmaker2022reward}).
The best known sample-complexity bound for the online access model is achieved by the computationally inefficient algorithm of \cite{zanette2020learning}, called \textsc{Eleanor}, which serves as a starting point of our work.

\definecolor{LightGreen}{rgb}{0.8,1,0.8}
\definecolor{LightRed}{rgb}{1,0.8,0.8}
\definecolor{LightBlue}{rgb}{0.68, 0.85, 0.9}
\begin{table}[t]
\centering
\begin{tabular}{|c|c|c|c|c|}
\hline
&\multicolumn{2}{c|}{Online RL}&\multicolumn{2}{c|}{Planning with simulator} \\
MDP class & $\poly(\cdot)$ sample & $\poly(\cdot)$ compute & $\poly(\cdot)$ sample & $\poly(\cdot)$ compute \\ \hline
Linear MDP & \multicolumn{4}{c|}{%
\citet{Jin_Yang_Wang_Jordan_2019}} \\\hline %
$q^\pi$-realizable MDP & 
\textbf{This work} & %
Open problem & \multicolumn{2}{c|}{%
\cite{yin2022efficient}} \\\hline
\end{tabular}
\medskip
\caption{Comparison of efficiency results for linear MDPs and $q^\pi$-realizable MDPs under online RL and planing with a simulator.
{\color{black}This work} establishes that $q^\pi$-realizable MDPs are also sample efficiently solvable under online RL.
The computational complexity of this problem remains {\color{black}open}.}
\label{tab:qpio}
\end{table}

In this work we consider the setting of linearly $q^\pi$-realizable MDPs.
As opposed to linear MDPs, before this work, sample efficient solutions 
were only known for this case when the MDP is accessed through a simulator
that implements some form of a state-reset function \citep{LaSzeGe19,yin2022efficient,weisz2022confident} (\cref{tab:qpio}).
In this work we 
resolve an open problem by \cite{Du_Kakade_Wang_Yan_2019}, and
show that having access to a state-reset is not essential in this setting. 
To this end, we present \qpieleanor (\cref{alg:main}) and a corresponding theorem (\cref{thm:main}) that shows that
\qpieleanor, which uses online interactions only, is
a provably sample-efficient solution to this problem.
The rest of this paper is organized as follows.
In \cref{sec:not} we introduce the basic definitions.
In \cref{sec:qpi-to-linear} we give an insight into the difference between linear $q^\pi$-realizability and linear MDPs, which motivates our approach.
In \cref{sec:alg} we describe our algorithm and the most important technical tools we discovered for its analysis.
Notably, in \cref{sec:auxiliary-real} we establish a rich structure inherent in $q^\pi$-realizable MDPs, which acts as the technical foundation to this work, and may be of independent interest.
Finally, \cref{sec:proof-overview} gives a summary of the proof of our main result (\cref{thm:main}),
 before concluding with some notes on future work in \cref{sec:future}.

\section{Preliminaries}\label{sec:not}

For a linear subspace $X$ of $\R^d$, let $\Proj_X$ denote the orthogonal projection matrix onto $X$. %
Throughout we fix $d \in \N^+$. For $L>0$, let $\B(L)=\{x \in \R^d: \norm{x}_2 \le L\}$ denote the $d$-dimensional Euclidean ball of radius $L$ centered at the origin, where $\|\cdot\|_2$ denotes the Euclidean norm.
Let $\pd$ denote the set of positive definite matrices in $\R^{d\times d}$.
We write $a\approx_\epsilon b$ for $a,b,\epsilon\in\R$ if $|a-b|\le\epsilon$.
Let $\one{B}$ be the indicator function of a boolean-valued (possibly random) $B$ taking value $1$ if $B$ is true and $0$ if false.
Let $\Dists(X)$ denote the set of probability distributions supported on set $X$. %
The rest of our notation is standard, but described in \cref{app:not} for completeness.

For the setting of episodic finite horizon RL, with horizon $H$,
a finite-action Markov decision process (MDP) describes an environment
for sequential decision-making.
It is defined by a tuple $(\cS,[\cA],P,\cR)$ as follows.
The state space $\cS$ is split across stages: $\cS=\left(\cS_t\right)_{t\in[H]}$ with $\cS_1=\{s_1\}$ for some designated initial state $s_1$.
Without loss of generality, we assume the $\left(\cS_t\right)_{t\in[H]}$ are disjoint sets.
We define the function $\stage:\cS\to[H]$ as $\stage(s)=t$ if $s\in\cS_t$.
We consider finite action spaces of size $\cA$ for some $\cA\in \N^+$, and without loss of generality, define the set of actions to be $[\cA]:=\{1,\dots,\cA\}$.
The transition kernel is $P:\left(\bigcup_{t\in[H-1]}\cS_t\right)\times[\cA]\to \Dists(\cS)$, with the property that transitions happen between successive stages, that is,
for any $t\in[H-1]$, state $s_t\in\cS_t$, and action $a\in[\cA]$, $P(s_t,a)\in\Dists(\cS_{t+1})$.
The reward kernel is $\cR:\cS\times[\cA]\to\Dists([0,1])$.
An agent interacts sequentially with this environment in an episode lasting $H$ steps by taking some action $a\in[\cA]$ in the current state.
The environment responds by transitioning to some next-state according to $P$, and giving a reward in $[0,1]$ according to $\cR$.%
\footnote{Here, the reward and next-state are independent, given the current state and last action. Independence is nonessential and is assumed only to simplify the presentation.}

We describe an agent interacting with the MDP by a \emph{policy} $\pi$,
which, to each history of interaction (including states, actions and rewards) assigns a probability distribution over the actions. Policies where this distribution only depend on the last state in the history are called \emph{memoryless}, and these are identified with elements of the set
$\Pi=\{\pi:\cS\to \Dists([\cA])\}$.
Using a policy $\pi$, starting at some state $s$ in an MDP induces a probability distribution over histories, which we denote by $\cP_{\pi,s}$.
For any $a\in[\cA]$, $\cP_{\pi,s,a}$ is the distribution over the histories when 
first action $a$ is used in state $s$, after which policy $\pi$ is followed.
$\E_{\bullet}$ is the expectation operator corresponding to a distribution $\P_{\bullet}$
(e.g., $\E_{\pi,s}$ is the expectation with respect to $\P_{\pi,s}$).
The state- and action-value functions $v^\pi$ and $q^\pi$ are defined as the expected total reward within the first episode while $\pi$ is used:
\begin{align*}
v^\pi(s)=\E_{\pi,s} \sum_{u=\stage(s)}^H R_u \quad\text{for }s\in\cS\,
\quad \text{ and }\quad
q^\pi(s,a)=\E_{\pi,s,a} \sum_{u=\stage(s)}^H R_u \quad\ \text{for }s\in\cS,\, a\in[\cA].
\end{align*}

Let $\pi^\star\in\Pi$ be an optimal policy, satisfying $q^{\pi^\star}(s,a)=\sup_{\pi\in\Pi}q^\pi(s,a)=\sup_{\pi\in\text{all policies}}q^\pi(s,a)$ for all $(s,a)\in\cS\times[\cA]$.
Let $q^\star(s,a)=q^{\pi^\star}(s,a)$ and $v^\star(s)=\sup_{a'\in[\cA]}q^\star(s,a)$ for all $(s,a)$. %

\section{From linear $q^\pi$-realizability to linear MDPs}\label{sec:qpi-to-linear}

As described in the introduction, we endow our MDP with a feature map $\phi:\cS\times[\cA]\to\B(\featurebound)$ for some $L_1>0$.
For reference, we start with a definition of linear MDPs with a parameter norm bound $\thetabound>0$, formalizing that the transition kernel and the expected rewards are approximately linear functions of the features:\footnote{Compared to the definition of \citet{jin2020provably}, our definition does not require
the existence of a vector-valued measure to represent the transition kernel.
This is a generalization that is compatible with all existing algorithms for linear MDPs.}
\begin{definition}\label{def:lin-mdp}[$\kappa$-approximately linear MDP]
For any $\kappa\le1$, an MDP 
is a $\kappa$-approximately linear MDP %
if {\em (i)} there exists 
$\theta_1,\dots,\theta_H \in \B(\thetabound)$ such that 
for any $h\in [H]$ and $(s, a)\in\cS_h\times[\cA]$,
$\abs{\E_{R\sim \cR(s,a)} R-\ip{\phi(s,a),\theta_h}}\le\kappa$
and {\em (ii)} for  any $f:\cS \to [0,H]$ and $h\in[H-1]$,
there exists
$\theta'_h\in \B(\thetabound)$ 
such that for all $(s, a)\in\cS_h\times[\cA]$,
$\abs{\E_{S'\sim P(s,a)} f(S')-\ip{\phi(s,a),\theta'_h}}\le\kappa$.
\end{definition}
A key consequence of the linear MDP assumption is that the \emph{inherent Bellman error}
\[\sup_{\theta_{h+1}\in\B(\thetabound)}\inf_{\theta_h\in\B(\thetabound)}
\sup_{(s,a)\in\cS_h\times[\cA]}
\Big|
\E_{
R\sim\cR(s,a), %
S'\sim P(s,a)
}%
R(s,a)+\max_{a'\in[\cA]}\ip{\phi(S',a'),\theta_{h+1}}
- 
\ip{\phi(s,a),\theta_h} 
\Big|
\,,
\]
scales with the misspecification $\kappa$.
This property is also referred to as the \emph{closedness to the Bellman operator}, and is a crucial component in the analysis of approximation errors for algorithms tackling linear MDPs.

In this work we consider a weaker linearity assumption where we only assume that the action-value functions are approximately linear:
\begin{definition}[$q^\pi$-realizability: uniform %
linear function approximation error of value-functions]
\label{def:q-pi-realizable}
Given an MDP, the uniform %
value-function approximation error (or misspecification) induced by
a feature map $\phi:\cS\times[\cA]\to\B(\featurebound)$,
over a set of parameters in $\B(\thetabound)$ is
\[
\eta = \sup_{\pi\in\Pi}%
\max_{h\in[H]} \inf_{\theta^{(h)}\in\B(\thetabound)} \sup_{(s,a)\in\cS_h\times[\cA]} \abs{ q^\pi(s,a) - \ip{\phi(s,a),\theta^{(h)}} }\,.
\]
For the MDP and the corresponding feature map, for all $h\in[H]$ fix any
$\theta_h:\Pi\to\B(\thetabound)$
mapping each memoryless policy $\pi\in \Pi$ to its ``parameter'',
such that %
\begin{align}\label{eq:near-realizable}
q^\pi(s,a) \approx_\eta \ip{\phio(s,a), \theta_h(\pi)}
\quad\quad\text{for all $\pi\in\Pi$, $s\in\cS_h$, and $a\in[\cA]$}
\,.
\end{align}
The set of all parameters $\Theta_h\subseteq\B(\thetabound)$ for a stage $h\in[H]$ is given by $\Theta_h = \left\{\theta_h(\pi)\,:\, \pi\in\Pi\right\}\,$.
\end{definition}
Note that $\theta_h$ satisfying \cref{eq:near-realizable} always exist \citep[Appendix C]{weisz2022confident}.
We focus on the feasible regime where $\eta$ is polynomially small in the relevant parameters.
Specifically, we assume that $\eta$ is bounded according to \cref{eq:eta-small-assumption}.
The main problem of interest in this work is the following:

\begin{problem}[informal]\label{prob:main}
For any $\epsilon, \zeta>0$ and any MDP with corresponding uniform %
value-function approximation error $\eta$,
derive an algorithm that,
with probability at least $1-\zeta$,
will
find an $\epsilon$-optimal policy (i.e., a policy $\pi$ such that $v^\pi(s_1)\ge v^\star(s_1)-\epsilon$)
by interacting with the MDP online for $T$ steps with $T$
bounded by a polynomial function of $(d,H,\epsilon^{-1},\log\zeta^{-1},\log \featurebound,\log\thetabound)$.
That the interaction with the MDP is online means that it is only possible
to observe the features corresponding to the current state, and to take an action and subsequently observe the resulting reward and next state, which then becomes the current state. We consider the fixed horizon episodic setting, that is, the next state is reset to the initial state $s_1$ after every $H$ steps.
\end{problem}

Algorithms developed for linear MDPs are not directly applicable to \cref{prob:main} when the MDP is only $q^\pi$-realizable:
While a linear MDP is also $q^\pi$-realizable, 
a $q^\pi$-realizable MDP may be neither a linear MDP, nor one with a low inherent Bellman error \citep{zanette2020learning}.
As an illustrative example, \cref{fig:qpi-lin}, left shows an MDP that is $q^\pi$-realizable but not linear.
To see this, observe that the features for both actions in $s_1$ are identical, but their transitions and rewards are not.
As illustrated in the figure however, if we \emph{skip} over the red states (with identical actions) by taking the first action on them and summing up the rewards received until we reach a black state, we arrive at a linear MDP.
This serves as the main intuition behind our work:
the red states have no bearing on action-values, so they can be skipped, and the resulting MDP is linear.

More generally, we can define the
\emph{range} of any state as the maximum possible difference in action-value that the choice of action in that state can make:
\begin{align}\label{eq:range-def}
\range(s)=\sup_{\theta\in\Theta_{\stage(s)}} \max_{i,j\in[\cA]} \ip{\phio(s,i,j), \theta}%
\text{ for all } h\in[H], s\in\cS_h\,,
\end{align}
where $\phio(s,i,j)=\phio(s,i)-\phio(s,j)$ is the notation for feature differences. Clearly, the choice of action in low-range states is not too important, as
\begin{equation}
v^\pi(s)-q^\pi(s,a) \le \range(s)+2\eta \quad\quad\text{ for any $\pi \in \Pi$ and all  $a\in[\cA]$.} \label{eq:vstar-vs-range}
\end{equation}
Not only are the action choices in low-range states unimportant for the task of finding a near-optimal policy for the MDP,
these choices can affect transitions and rewards in a nonlinear way.
Interestingly, the existence of low-range states is the reason why $q^\pi$-realizable MDPs are not necessarily linear, as shown by the next result (proved in \cref{proof:norange-linear-mdp}), which follows easily from \cref{lem:admissible-realizability}.
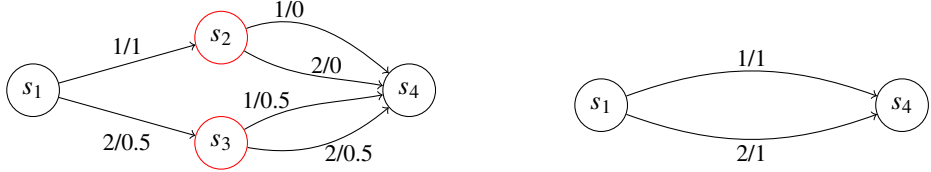
\begin{figure}
  \begin{subfigure}{0.5\textwidth}\centering
\begin{tikzpicture}[node distance=2.5cm,
fatnode/.style={draw,circle, minimum size=7mm}]
  \node[fatnode] (start) {$s_1$};
  \node[fatnode] [right of=start,yshift=0.7cm,draw=red] (child1) {$s_2$};
  \node[fatnode] [right of=start,yshift=-0.7cm,draw=red] (child2) {$s_3$};
  \node[fatnode] [right of=child2,yshift=0.7cm] (end) {$s_4$};

  \draw[->] (start) -- (child1) node[midway,above] {\small 1/1};
  \draw[->] (start) -- (child2) node[midway,below,yshift=-0.1cm] {\small 2/0.5};
  \draw[->] (child1) to [out=20,in=140] (end) node[midway,above,xshift=3.4cm,yshift=0.85cm] {\small 1/0};
  \draw[->] (child1) to [out=-30,in=170] (end) node[midway,above,xshift=3.9cm,yshift=0.1cm] {\small 2/0};
  \draw[->] (child2) to [out=30,in=-170] (end) node[midway,below,xshift=3.1cm, yshift=0.1cm] {\small 1/0.5};
  \draw[->] (child2) to [out=-10,in=-140] (end) node[midway,below,xshift=4.2cm, yshift=-0.6cm] {\small 2/0.5};
\end{tikzpicture}
  \end{subfigure}%
  \hspace*{\fill}   %
  \begin{subfigure}{0.5\textwidth}\centering
\begin{tikzpicture}[node distance=1.5cm,
fatnode/.style={draw,circle, minimum size=7mm}]
  \node[fatnode] (start) {$s_1$};
  \node[fatnode] [right of=child2,yshift=0.7cm] (end) {$s_4$};

  \draw[->] (start) to [out=340,in=200] (end) node[midway,above,xshift=2cm,yshift=0.4cm] {\small 1/1};
  \draw[->] (start) to [out=20,in=160] (end) node[midway,below,xshift=2cm,yshift=-0.4cm] {\small 2/1};
\end{tikzpicture}
  \end{subfigure}%
\caption{\textbf{Left:}
MDP with deterministic transitions and rewards (edges are labeled with action/reward).
\textbf{Right:}
The same MDP with the red ``low-range'' states ``skipped'' over.
$\phi(s_1,\cdot)=(1), \phi(s_3,\cdot)=(0.5),\phi(\cdot,\cdot)=(0)$ otherwise.
Both MDPs are $q^\pi$-realizable, but only the right MDP is linear.}  \label{fig:qpi-lin}
\end{figure}

\begin{proposition}\label{prop:norange-linear-mdp}
Consider an MDP with uniform value-function approximation error $\eta \ge 0$.
If there are no states $s\in\cS$ with $\range(s)<\alpha$ for some $\alpha>0$, then the transitions and rewards of the MDP are linear (\cref{def:lin-mdp}) with misspecification scaling with $\eta$, and parameter norms scaling inversely with $\alpha$.
\end{proposition}

\paragraph{Our approach.} The above result immediately offers a strategy to learn under the (linear) $q^\pi$-realizability assumption. Assuming access to an oracle that can determine whether or not $\range(s)<\alpha$ for any state $s$, the MDP could be ``converted'' to one that has no low-range states but has near-identical state and action-value functions of any policy (compared to the original MDP), by skipping over low-range states %
(by executing an arbitrary action) until a state with a range at least $\alpha$ is reached.
We will call such a multi-state transition a \emph{skippy step} and refer to such a policy as a \emph{skippy policy}.
The reward presented for a skippy step is the cumulative reward over the skipped states. 
When the oracle is correct, the new MDP is a linear MDP, allowing techniques such as \textsc{Eleanor} to efficiently learn a near-optimal policy.
This conversion argument is part of the intuition of our method, but it is not strictly part of the proof,
so we defer the details to \cref{app:skipconvert}.
The only missing piece for solving the general case, \cref{prob:main}, is learning an oracle that can suggest when to skip over a state, and combining it with the learning algorithm for the linear MDP.
This general approach leads to our algorithm, \qpieleanor, 
which runs a modified version of \textsc{Eleanor} with guessed oracles. During the algorithm, we detect when an incorrect oracle leads to suboptimal results, and refine the oracle accordingly. The details of the algorithm are explained in the next section.

\section{Algorithm}\label{sec:alg}

In this section we present our main results following our plan outlined above. We first give
\cref{alg:main}, along with
a high-level overview of the algorithm; the details are explained throughout the section.
The parameters of the algorithm are presented in \cref{app:alg-params}.

\begin{algorithm}[ht]
\caption{
\qpieleanor
}\label{alg:main}
\begin{algorithmic}[1]
\State \textbf{Input:} accuracy $\epsilon>0$, failure probability $\zeta>0$
\State Initialize $m\gets0,$ $m'\gets0$, $\Q_h=\thetabound I$
for $h\in[H]$, $\pi^0=(s\mapsto 1)$
\While{$m'\le \mpmax$}\label{line:while}
	\State $m\gets m+1$, $m'\gets m'+1$ \Comment{$m'$ also counts iterations repeated due to Line~\ref{line:redo2}}
	\State Estimate optimistic problem parameters $\hat G, \bar\theta$ by solving \cref{def:opt-problem} \label{line:opt}
	\For{$k\in[H]$} \label{line:k-for} %
		\State Let $\pi^{mk}$ be the policy defined by $\skippypolicy(\hat{G},\bar\theta,k\!)$
		\State Sample $n$ trajectories by executing
		 $\pi^{mk}$ from $s_1$ for $n$ episodes\!\! \label{line:execute-traj}
		 \State Record data $(S^{mkj}_h,A^{mkj}_h,R^{mkj}_h)_{h \in [H], j \in [n]}$ and stage-mapping functions $(p^{mkj})_{j \in [n]}$ \label{line:stage-map-def} %
	\EndFor
 	\State Solve \cref{def:consistency-opt} with input $(\hat G,\bar\theta)$,
	 \label{line:consopt}\Comment{Consistency check}
	\Statex\hspace{0.46cm} record its value $x$ (maximum discrepancy), and arguments $v$ (direction) and $i$ (stage). 
	\State Calculate useful component $w\gets \Proj_{Z(\Q,i)} v$ \Comment{\cref{def:valid-preconditions}}
	\If{$x>\text{discrepancy\_threshold}$} %
		\State $\Q_i\gets \left(\Q_i^{-2}+\Q_i^{-1}ww^\top\Q_i^{-1}\right)^{-\frac12}$ \label{line:decoydetect} \Comment{append $\Q_i^{-1}w$ to $C_i$ according to \cref{eq:q-form}}
		\State $m\gets m-1$ \Comment{redo this iteration}\label{line:redo2}
		\State \textbf{continue}
	\EndIf
	\If{$\text{average\_uncertainty} \le \text{uncertainty\_threshold}$} \label{line:cons-check-pass} %
		\State \Return policy $\pi^{mH}$ \label{line:return}
	\EndIf
\EndWhile
\end{algorithmic}
\end{algorithm}

\begin{algorithm}[ht]
\caption{
\skippypolicy
}\label{alg:skippy}
\begin{algorithmic}[1]
\State Input: $\hat G, \bar\theta,k$
	\State Initialize $S_1\gets s_1$, $j\gets 1$, $\pi^0\gets (s\mapsto 1)$, stage mapping $p$
	\For{$i=1$ to $H$}
		\State Compute skip probabilities $\tau_i\gets\tau_{\hat G}(S_i)$ and non-skip action $a^+\gets \pi^+_{\bar\theta}(S_i)$ from \cref{eq:tau-pi+-def}
		\State Sample independently $B_i \sim \text{Bernoulli}(\tau_i)$
		\If{$B_i=0$}
		    $A_i\gets 1$ \Comment{skip (follow $\pi^0$) with probability $1-\tau_i$} %
		\Else
			\State $p(j)\gets i$, $j\gets j+1$
			\If{$j\le k$}
				$A_i\gets a^+$
				(Phase I)
				\textbf{else}
				$A_i\gets 1$
				(Phase II) %
			\EndIf
		\EndIf
		\If{$i=H$}
			\State $p(j')=H+1$ for $j'=j,\dots,H$
		\EndIf
	\EndFor

\end{algorithmic}
\end{algorithm}

For every stage $h \in [H]$, the algorithm keeps a progressively refined estimate of the geometry of the parameter space $\Theta_h$, by maintaining an ever shrinking ellipsoid enclosing $\Theta_h$. This ellipsoid is parametrized by an 'inverse covariance matrix'-like quantity $\Q_h$, determined by $\ordot(d)$ vectors, which guarantees $\max_{\theta_h \in \Theta_h} \norm{\theta_h}_{\Q_h^{-2}} = \ordot(\sqrt{d})$. Looking at the definition of $\range$ in \cref{eq:range-def}, it is clear that the smaller the ellipsoid becomes, the better estimate we can give for the ranges.

Given some data collected so far and $(Q_h)_{h\in[H]}$, \qpieleanor computes optimistic estimates of the action-values by calculating an optimistic policy parameter $\bar\theta$, as well as a guess $\hat G$ to a near-optimal design which is used to estimate the range for the states (due to technical reasons, $\hat G$ will guess a near-optimal design for the transformed parameter space $Q_h^{-1} \Theta_h$).

Data is collected by running stochastic versions of skippy policies on the MDP, where the states to be skipped over are determined based on the range estimates; when a state is skipped, an action is selected using a deterministic policy $\pi^0$ that always chooses the first action in every state. To ensure that the estimation problem is smooth in terms of $\hat G$, we use a smoothed version of skippy policies, where states are skipped randomly, and the probability of skipping is larger for states with lower ranges, while high-range states are never skipped.
Similarly to \textsc{Eleanor}, we aim to estimate the action-value function of a state-action pair by adding the estimated one-step reward to the estimated value-function of the next state. However, unlike \textsc{Eleanor}, we would like to do this in the reduced MDP, where the low-range states that are skipped over are removed (and the corresponding transitions are replaced by skippy steps). Since we do not know these states in advance, we run exploratory policies that skip over next states starting from any state: 
namely, we run $\skippypolicy(\hat G, \bar\theta,k)$ for all $k \in [H]$ with a maximum number of unskipped states $k$ (Phase I), and once this is skip budget is exhausted, all remaining states are skipped over by rolling out $\pi^0$ (Phase II), which ensures that we collect enough data at every stage of the MDP to be able to estimate the one-skippy-step reward of any skipping mechanism.
Compared to \textsc{Eleanor}, this introduces an additional loop in Line~\ref{line:k-for} of \qpieleanor; see \cref{app:eleanor-perspective} for additional details.
For any execution, \skippypolicy maintains a stage-mapping function $p$, which, for any stage $h$ of the trajectory in the reduced MDP %
gives the stage index in the original MDP. In other words, $p(j)$ is the stage of the landing state of the $j^{th}$ skippy step.

Finally, we check if the data collected is consistent with our estimates $\hat G$ and $\bar\theta$, by calculating the maximal discrepancy of the estimates of the action-value difference at the last non-skipped state of $\pi^{mk}=\skippypolicy(\hat{G}, \bar\theta,k)$ and that of the fixed skipping policy $\pi^0$ in different directions in the parameter space. If the discrepancy is too large for any $k$, we add the discrepancy-maximizing direction to $\Q$ and throw away the data collected in this (i.e., the $m^{th}$) iteration; this is achieved by reducing the iteration counter $m$ by 1. On the other hand, if the discrepancy is small enough, we can guarantee that the gap between the value of $\pi^{mH}$ %
and $v^\star(s_1)$
scales with how much new information we collected, thus the algorithm can terminate returning this policy if this term is sufficiently small (which it eventually has to be).

The following theorem shows that with high probability, \qpieleanor finds a near-optimal policy after polynomially many interactions with the MDP. Rhe proof sketch is provided in \cref{sec:proof-overview}, while our method and proof strategy is explained from the perspective of \textsc{Eleanor} in \cref{app:eleanor-perspective}.
\begin{theorem}\label{thm:main}
With probability at least $1-\zeta$,
\qpieleanor interacts with the MDP for at most
$\ordot\left(H^{11}d^7/\epsilon^2\right)$ many steps,
before returning a policy $\pi$ that satisfies
$v^\star(s_1) \le v^{\pi}(s_1)+\epsilon$.
\end{theorem}

\subsection{Preconditioning: the enclosing ellipsoid}\label{sec:precond}\label{sec:range}

In this section we give the technical details about the effects of using the matrix $Q_h$ describing an enclosing ellipsoid for $\Theta_h$ (see \cref{lem:theta-precond-norm-bound}) as preconditioning the features.

\begin{definition}[Valid preconditioning]\label{def:valid-preconditions}
$Q=(\Q_h)_{h\in[H]}$ is a valid preconditioning matrix sequence if
for all $h\in[H]$
\begin{align}\label{eq:q-form}
\textstyle
\Q_h=\left(\miota I
+\sum_{v\in C_h} vv^\top
\right)^{-1/2}
\end{align}
for some
sequence $C_h=(v_1,\dots,v_n)$ of vectors in $\R^d$ %
such that for all $1\le i\le n$,
\begin{align}\label{eq:precond-lownorm}
\textstyle
\sup_{\theta\in\Theta_h} \left|\ip{\theta, v_i}\right| \le 1 %
\quad\text{and}\quad
\norm{\left(\miota I + \sum_{j=1}^{i-1}v_j v_j^\top\right)^{-\frac12} v_i}_2^2\ge \frac12
\quad\text{and}\quad
\norm{v}_2\le \precondbound\,,
\end{align}
where %
$\precondbound$ is some fixed
polynomial of the problem parameters
$(d,H,\epsilon^{-1},\log\zeta^{-1},\log \featurebound,\log\thetabound)$.
(see \cref{eq:precondbound-choice} for its precise value).

For a valid preconditioning $\Q$ and some $h\in[H]$,
let $Z(\Q,h)$ be the linear subspace spanned by those eigenvectors of $Q$ whose corresponding eigenvalues are at least $\precondbound^{-2}$.
Let $\Proj_{Z(\Q,h)}$ be the orthogonal projection matrix onto this subspace.
\end{definition}

Sometimes it will be convenient to \emph{precondition} the features and parameters so that the enclosing ellipsoid is transformed to a ball of controlled radius (as \cref{lem:theta-precond-norm-bound} will show).
To this end, introduce for all $h\in[H]$ and $(s,a,b)\in\cS_h\times[\cA]\times[\cA]$ the following:\footnote{Note that $Q_h, h \in [H]$ is invertible by construction.}
\begin{align}\begin{split}\label{eq:defs-related-to-q}
\phi_\Q(s,a)=\Q_{h}\phi(s,a),\qquad
\phi_\Q(s,a,b)=\Q_{h}\phi(s,a,b) \\
\theta_h^\Q(\pi)=\Q_h^{-1} \theta_h(\pi),\qquad
\Theta_h^\Q=\left\{\theta_h^\Q(\pi) \,:\, \pi\in\Pi  \right\} = \left\{\Q_h^{-1} \theta\,:\, \theta \in \Theta_h\right\} \\ %
\hat q^\pi(s,a)=\ip{\phi(s,a),\theta_h(\pi)} = \ip{\phi_\Q(s,a), \theta_h^\Q(\pi)}
\text{ for all } \pi\in\Pi\,.
\end{split}\end{align}

The next lemma (proved in \cref{proof:theta-precond-norm-bound}) shows that for all $h \in [H]$, $\Q_h$ defines an enclosing ellipsoid for $\Theta_h$; that is, $\Theta_h \subset \{\theta: \norm{\theta}_{\Q_h^{-2}} \le \sqrt{d_1+1}\}$.
\begin{lemma}\label{lem:theta-precond-norm-bound}
Let $d_1\!=\! 4d\log(1\!+\!16\precondbound^4\thetabound^4)\!=\!\ordot(d).$ Then,
for any valid preconditioning $\Q$ and $h\in[H]$,
\[
\textstyle
\sup_{\theta\in\Theta_h}\norm{\theta}_{\Q_h^{-2}}
=\sup_{\theta\in\Theta_h^\Q}\norm{\theta}_2\le \dfour \,.
\]
\end{lemma}

\vspace{-2mm}Clearly, every time a new vector is added to $C_h$, the enclosing ellipsoid $\{\theta: \norm{\theta}_{\Q_h^{-2}} \le \sqrt{d_1+1}\}$ 
shrinks (as a positive semidefinite matrix is added to $\Q_h^{-2}$). %
The following lemma (also proved in \cref{proof:valid-precond-d1-limit}) uses an elliptical potential argument to bound the number of times this can happen. %
\begin{lemma}\label{lem:valid-precond-d1-limit}
For any valid preconditioning $\Q$, for all $h\in[H]$, the length of sequence $C_h$ corresponding to $\Q_h$ according to \cref{def:valid-preconditions} is at most $d_1$.
\end{lemma}

\vspace{-2mm}\paragraph{Near-optimal design for $\Theta^\Q_h$.} As $\Q_h$ only provides an enclosing ellipsoid for $\Theta_h$, we introduce an (unknown) ellipsoid that aligns better with $\Theta^\Q_h$.
For all $h\in[H]$, fix a set $G_h^\Q$ of policies of size $d_0:=4d\log\log(d)+16$, together with a probability distribution $\rho_h^\Q$ on $G_h^\Q$,
such that $(G_h^\Q, \rho_h^\Q)$ is a near-optimal design for $\Theta_h^\Q$ (i.e., satisfying
\cref{def:nearopt}).
The existence of such a near-optimal design follows from \cite[Part (ii) of Lemma 3.9]{todd2016minimum}.

We apply $G_h^\Q$ to define a cruder version of $\range$ that depends only on a small set of policies, and can therefore be succinctly parametrized to inform \skippypolicy:
\begin{align}\label{eq:rangeq-def}
\range_\Q(s)=\max_{\pi\in G_h^\Q} \max_{i,j\in[\cA]} \ip{\phio(s,i,j), \theta_h(\pi)}\quad\quad\text{for all } h\in[H], s\in\cS_h\,.
\end{align}
\vspace{-2mm}$\range_\Q$ is easy to estimate, and can be used to bound the $\range$ function (proved in \cref{proof:opt-design-gap-and-vstar-gap}):
\begin{proposition}\label{prop:opt-design-gap-and-vstar-gap}
For all %
$s\in\cS$ and $\Q\in\pdseq$, %
$\range(s) \le \sqrt{2d}\range_\Q(s)$.  %
\end{proposition}

\subsection{Linearly realizable functions}\label{sec:auxiliary-real}

\newcommand{\bu}{{\mathbf{u}}}
\newcommand{\bff}{{\mathbf{f}}}
\newcommand{\bw}{{\mathbf{w}}}
\newcommand{\bG}{{\mathbf{G}}}
\newcommand{\bbt}{\bar{\Theta}}
\newcommand{\bF}{{\mathbf{F}}}
\newcommand{\qh}{{\hat q}}

$q^\pi$-realizability (\cref{def:q-pi-realizable}) implies the linearity of many more functions than the action-value functions. In this section we characterize an interesting set of such functions, whose (approximate) linearity plays a crucial role in our algorithm and analysis, as their parameters can be conveniently estimated by least squares using the features.
We rely on functions $f:\cS_h\to\R$ (for some $h\in[H]$) being small for all states, relative to the states' $\range_\Q$-value:
\begin{definition}\label{def:alpha-admissible}
For any $h\in[H]$,
$f:\cS_h\to\R$ is $\alpha$-admissible for some $\alpha>0$ if for all $s\in\cS_h$,
$|f(s)|\le \range_\Q(s)/\alpha$. %
\end{definition}

\vspace{-1mm}The key observation is that expected (admissible) $f$ values are linearly realizable. %
\begin{lemma}[Admissible-realizability]\label{lem:admissible-realizability}
If $f:\cS_h\to\R$ is $\alpha$-admissible
then it
is realizable, that is, for all $t\in[h-1]$
and $\pi\in\Pi$,
there exists some $\tilde\theta\in\R^d$ with $\norm{\tilde\theta}_2\le 4d_0\thetabound/\alpha$ such that for all $(s,a)\in\cS_t\times[\cA]$,
\[
\E_{\pi,s,a}f(S_h)\approx_{{\eta_0}}\ip{\phi(s,a), \tilde\theta}\quad\quad\quad\quad\text{where ${\eta_0}=5d_0\eta/\alpha$.} %
\]
\end{lemma}
\vspace{-2mm}The proof relies on constructing a set of policies that at states $s\in\cS_h$ take a higher value action as opposed to a lower one with a certain probability, configured such that
the expected action-value difference of some pairs within the set of policies is (approximately) proportional to $f(s)$.
Thus, a linear combination of the action-values of policies in this set are also (approximately) proportional to $f(s)$.
The statement of the lemma then follows from setting $\tilde\theta$ to the corresponding linear combination of the policies' parameters.
The full proof is presented in \cref{proof:admissible-realizability}.

Next, we define matrix-valued functions with a special admissibility guarantee even when the underlying scalar-valued function does not satisfy any non-trivial admissibility criterion.
We introduce a \emph{guess} on the near-optimal design parameters that define $\range_\Q$ (\cref{eq:rangeq-def}) for some valid preconditioning $\Q$:
\begin{definition}
For $h\in[2:H]$,
fix some arbitrary order of the policies in the set $G_h^\Q$ (recall that this set is the support of the near-optimal design for $\Theta_h^\Q$).
Let the parameter of the $i^\text{th}$ policy in $G_h^\Q$ be $\vartheta_h^i$ %
for $i\in [d_0]$. %
Call a ``guess'' of these parameters
$\hat G=(\hat G_h)_{h\in[2:H]}=(\hat\vartheta^i_h)_{h\in[2:H],i\in[d_0]}$ ``valid'', if for all $h\in[2:H],i\in[d_0]$,
$\hat\vartheta^i_h\in\B(\dfour )$.
Let the set of valid guesses be $\bG$.\footnote{Note that while $G_h^\Q$ contains policies, $\bG$ and its elements (commonly denoted by $\hat G$) contain policy parameter vectors.}
By \cref{lem:theta-precond-norm-bound}, $(\vartheta_h^i)_{h\in[2:H],i\in[d_0]}\in\bG$, that is, it is a valid guess, and we call this the ``correct'' guess. %
\end{definition}
From a guess $\hat G=(\hat\vartheta^i_h)_{h\in[2:H],i\in[d_0]}$ we can calculate corresponding guesses of the $\range_\Q$-values:
\[
\range_\Q^{\hat G}(s)=\max_{k\in[d_0]}\max_{i,j\in[\cA]} \ip{\phi_\Q(s,i,j),\hat\vartheta^k_{\stage(s)}}\quad\quad \text{for all } h\in[2:H], s\in\cS_h\,.
\]
Note that for any $h\!\in\![2\!:\!H]$ and $s\!\in\!\cS_h$, $\range_\Q^{\hat G}(s)=\range_\Q(s)$ if $\hat G$ is the correct guess for stage $h$.

Let $\bar{\phi}_\Q(s)$ be the unit vector in the direction of the largest feature difference between actions in $s$ and the zero vector if all feature vectors are the same (see \cref{eq:phiwitness} for a formal definition).
Then, for any $\hat G\in\bG$, $h\in[2:H]$, and $f:\cS_h\to[-H,H]$,
let
\[
\bff(s)=\bar{\phi}_\Q(s){\bar{\phi}_\Q(s)}^\top \min\Big\{1,\range_\Q^{\hat G}(s) \frac {\sqrt{2d}H}\epsilon\Big\} f(s) \quad \text{for } s\in\cS_h\,.
\]
\vspace{-0mm}For such $\bff:\cS_h\to \R^{d\times d}$, we adopt the notation $a^\top \bff b$ for any $a,b\in\R^d$
to denote the function $s\in\cS_h\mapsto a^\top \bff(s) b$,
and similarly, $\trace(\bff)$ to denote the function $s\in\cS_h\mapsto \trace(\bff(s))$.

Let $\Proj_{\pa(\Q,h)}$ be the projection matrix onto the linear subspace spanned by those eigenvectors of the design matrix $V(G_h^\Q, \rho_h^\Q)$ (defined in \cref{eq:design-matrix})
whose corresponding eigenvalues are at least $\gamma$ (for some $\gamma>0$ specified in \cref{app:alg-params}).
Intuitively, this is the subspace where
$\Theta^\Q_h$ has a sufficiently large width.
Let $\Proj_{\perp(\Q,h)}$ be the projection to the orthogonal complement subspace.
For any $v\in\R^d$, we write $v_{\pa(\Q,h)}$ and $v_{\perp(\Q,h)}$ for $\Proj_{\pa(\Q,h)} v$ and $\Proj_{\perp(\Q,h)} v$, respectively. %

We are now ready to state our special admissibility guarantee, which is proved in \cref{proof:dual-admissibility}.
Let $\alpha=\ordot(\epsilon/(d^{1.5}H^2))$ be as in \cref{eq:alpha-def}.
\begin{lemma}\label{lem:dual-admissibility}
For any $h\in[2:H]$, %
$\hat G\in\bG$,
any function $\bff$ constructed as above from some $f:\cS_h\to[-H,H]$,
and any $v,w\in\B(1)$,
$v^\top_{\pa(\Q,h)} \bff w$ is $\alpha$-admissible. %
Furthermore, if $\hat G=(\vartheta_h^i)_{h\in[2:H],i\in[d_0]}$ (the correct guess),
$\trace(\bff)$ is also $\alpha$-admissible.
\end{lemma}

\subsection{Least-squares targets and \cref{def:opt-problem}}
\label{sec:lse-targets}
Recall that \qpieleanor estimates
action-values of states by first adding the estimated one-step reward and the estimated value-function of the next state in the reduced MDP (where low-range states are skipped). 
Due to the linearity of $q^\pi$-values, these can be used 
as target variables of a least-squares estimator
to estimate the policy parameters.
This estimator is only guaranteed to be accurate if
the right (low-range) states are skipped; %
otherwise, we will argue in \cref{subsec:cons-check} that a discrepancy is detected and it is handled by changing the preconditioning $\Q$. 
Finally, to ensure optimism, we select parameter estimates that lead to the largest estimated policy values. The whole estimation process leads to \cref{def:opt-problem}, which we define in this section
along with the functions that it uses as least-square targets.
Each estimation is for a particular stage $h$ and may use the estimates $\bar\theta_i$ of \cref{def:opt-problem} for stages $i>h$.
In this subsection, we consider the $m^\text{th}$ iteration of the optimization called by \qpieleanor, and consider $Q$ fixed.
As a shorthand, we introduce the following notation for $l\in[m],j\in[n],k\in[H]$:
\begin{align*}
&\psuper{lkj}{k}=p^{lkj}(k) \quad\quad\text{as recorded in Line~\ref{line:stage-map-def} of \cref{alg:main}, and}\\
&S^{lkj}_{\pshort{k}}=S^{lkj}_{p^{lkj}(k)},\; A^{lkj}_{\pshort{k}}=A^{lkj}_{p^{lkj}(k)},\; R^{lkj}_{\pshort{k}}=R^{lkj}_{p^{lkj}(k)},\;
\phi^{lkj}_t=\phi(S^{lkj}_t,A^{lkj}_t),\;
\phi^{lkj}_{\pshort{k}}=\phi(S^{lkj}_{\pshort{k}},A^{lkj}_{\pshort{k}})\,.
\end{align*}
We collect the set of $(l,k,j)$ tuples for which the $k^\text{th}$ skippy step lands at stage $t$, for $t\in[H]$, as
\[
\bI^m(t)=\left\{(l,k,j):\, l\in[m-1],j\in[n],k\in[H], \psuper{lkj}{k}=t\right\}
\]
Note in particular that here $l\in[m-1]$, so $\bI^m$ only considers data collected prior to iteration $m$.

To estimate the parameters $\hat{G}$ and $\bar{\theta}$, we consider (simulated) trajectories
of \skippypolicy starting from stage $t$. For simplicity, we suppress the dependence of quantities on $\hat{G}$ and $\bar{\theta}$, which will be brought back later. %
The skipping probability $1-\tau$, the policy $\pi^+$ (to be also used in \skippypolicy), and corresponding clipped action-value estimates are defined as
\begin{align}
\begin{split}\label{eq:tau-pi+-def}
\tau(s)&= \min\left\{1,\range_\Q^{\hat G}(s)\frac {\sqrt{2d}H}\epsilon\right\}\quad\quad\text{if }\stage(s)>1\text{, and }\tau(s_1)=1;\\
\pi^+(s_i)&= \argmax_{a\in[\cA]}\ip{\phi(s_i,a),\bar\theta_i},%
\quad\quad\quad\quad
C(s_i) = \clip_{[0,H]} \ip{\phi(s_i,\pi^+(s_i)),\bar\theta_i}.\\
\end{split}
\end{align}
Let $\traj{i}\,=(s_i,a_i,r_i,\dots,s_H,a_H,\sr)\in\cS_i\times[\cA]\times[0,1]\times\dots\times[0,1]$ be any ending of a trajectory.
For $\traj{t+1}$, let $I$ be the (random) index of the first state that is \emph{not} skipped by \skippypolicy with the above $\tau$ (or $H+1$, if such an index does not exist). Then the estimated policy value of \skippypolicy from stage $t$ is
\vspace{-1mm}
\[
\textstyle\E_I [\sum_{u=t}^{I-1}r_u + \I{I<H+1}C(s_I)]\,,
\]
the sum of rewards along the skipped states plus the policy-value estimate from stage $I$.
It follows from \cref{cor:F-components-realizable} below (proved based on \cref{lem:dual-admissibility})
that if $\range_\Q^{\hat G}$ is an accurate estimate of $\range_\Q$, then
this quantity decomposes into terms that are
 linearly expressible using the features. Therefore, we use such quantities as least-square targets.
Indeed, writing out the expectation, we can re-express the estimated policy value as the sum of all rewards $\sum_{u=t}^{H}r_u$ plus a correction term $E^\to(\traj{t+1})$ defined as
\begin{align}\label{eq:fact:E-to-probab-form}
\begin{split}
E^\to(\traj{i})=\sum_{j=i}^{H} D(\traj{j})\tau(s_j)\prod_{j'=i}^{j-1}(1-\tau(s_{j'})) \,\text{ where }\,
D(\traj{i})&= C(s_i) - \sum_{u=i}^{H} r_u \quad\text{for $i>1$.}
\end{split}
\end{align}

\vspace{-2mm}The next optimization problem aims to find optimistic parameters yielding the largest estimated action-value function for $s_1$, where $\bar{\theta}$ is in the confidence ellipsoid of the least-squares estimates $\hat\theta$.

\begin{optproblem}[for iteration $m$]\label{def:opt-problem}
For input state $s$, with $\beta$ defined in \cref{app:alg-params} (emphasizing the dependence of functions defined above on $\hat G$ and $\bar\theta$ by adding them as subscripts):
\begin{align*}
\argmax_{%
\hat G\in\bG, \bar\theta_t\in\B(4d_0H\thetabound/\alpha) \text{ for } t\in[H]} %
\,\, %
C_{\hat G \bar\theta}(s_1)
\qquad\text{subject to, for all }t\in[H] \hspace{3cm}\\
\small {X_{mt}= \lambda I+ \!\!\!\!\sum_{lkj\in \bI^m(t)} \!\!\!\!\phi^{lkj}_t {\phi^{lkj}_t}^\top,
\;\norm{\bar\theta_t-\hat\theta_t}_{X_{mt}}\!\!\le \optnormconst,
\;
\hat\theta_t=X_{mt}^{-1}\!\!\!\! \sum_{lkj\in \bI^m(t)} \!\!\!\!\phi^{lkj}_t \Big(\underbrace{E^\to_{\hat G \bar\theta}(S^{lkj}_{t+1},\dots,\Sr^{lkj}) + \sum_{u=t}^{H} R^{lkj}_u}_{\text{least-squares target}} \Big)}.
\end{align*}
\end{optproblem}

\vspace{-3mm}Since our realizability results in \cref{sec:auxiliary-real} only apply to functions defined at a given stage 
(as only memoryless policies are $q^\pi$-realizable),
to be able to show that the least-squares targets are linearly realizable, we first decompose $E^\to(\traj{i})$ ($i\in[2:H]$) to directly express the effect of each stage in the trajectory (backwards): defining $E(\traj{i})=E^\to(\traj{i}) - E^\to(\traj{i+1})$ (for convenience, we use the notation $E^\to(\traj{H+1})=0$), we easily obtain
\vspace{-1mm}
\begin{align}\label{eq:E-def}
\textstyle E^\to(\traj{i})=\sum_{j=i}^{H} E(\traj{j})
\quad\quad&\text{and}\quad\quad
E(\traj{i})=\tau(s_i)\big(D(\traj{i})-
E^\to(\traj{i+1})
\big). %
\end{align}

\vspace{-2mm}Next we define matrix-valued functions, whose trace equals $E(\traj{i})$, that have the same form as $\bff$ in \cref{sec:auxiliary-real}, for which \cref{lem:dual-admissibility} applies.
This is crucial in establishing optimism of \cref{def:opt-problem}, as well as learning from instances where we detect that $E^\to$ is not realizable in \cref{def:consistency-opt}. To this end, let
\begin{align*}
F(\traj{i})=\bar{\phi}_\Q(s_i){\bar{\phi}_\Q(s_i)}^\top E(\traj{i})
\quad\quad&\text{and}\quad\quad \textstyle
\bar F(s_i)=\E_{\pi^0,s_i} [F(s_i,A_i,\dots,\Sr)] \quad\text{for } s_i\in\cS_i\,.\nonumber
\end{align*}

\vspace{-2mm}Let $\bbt=\left(\B(4d_0H\thetabound/\alpha)\right)^H$ denote the base set for the variables $\bar\theta_t$ in \cref{def:opt-problem}. 
As $\bar F$ is of the same form as $\bff$, we can apply \cref{lem:dual-admissibility} and then \cref{lem:admissible-realizability} to arrive at the following:
\begin{corollary}\label{cor:F-components-realizable}
For any $\hat G\in\bG$, $\bar\theta\in\bbt$, $v,w\in\B(1)$, and
for any $t\in[H-1]$, $i\in[t+1:H]$,
there exists some $\tilde\theta_{ti}\in\R^d$ with
$\norm{\tilde\theta_{ti}}_2\le 4d_0\thetabound/\alpha=1/\sqrt{\lambda}$ such that
for all $(s,a)\in\cS_t\times[\cA]$, where $\eta_0$ is defined in \cref{lem:admissible-realizability}.
\begin{align}\label{eq:barF-realizable}
\textstyle
\E_{\pi^0,s,a} \left[v_{\pa(\Q,i)}^\top\bar F_{\hat G \bar\theta}(S_i)w\right] \approx_{\eta_0} \ip{\phi(s,a),\tilde\theta_{ti}}\,.
\end{align}

\vspace{-3mm}Furthermore, if $\hat G$ is the correct guess,
there exists some $\tilde\theta'_{ti}\in\R^d$ with
$\norm{\tilde\theta'_{ti}}_2\le 4d_0\thetabound/\alpha$ such that
for all $(s,a)\in\cS_t\times[\cA]$,
$\E_{\pi^0,s,a}[E_{\hat G \bar\theta}(S_i,\ldots,R_H))] = \E_{\pi^0,s,a} [\trace(\bar F_{\hat G \bar\theta}(S_i))] \approx_{\eta_0} \ip{\phi(s,a),\tilde\theta'_{ti}}$.
\end{corollary}

\subsection{Checking consistency}\label{subsec:cons-check}

Considering the $m^\text{th}$ iteration of \qpieleanor, we want to verify if the estimated targets of \cref{def:opt-problem} are accurate (and learn if a discrepancy is detected),
by using \cref{cor:F-components-realizable} on the targets' decomposition into $F$-functions.
We filter the data collected in the $m^\text{th}$ iteration
with the indicator\vspace{-1mm}
$\tilde c^j_{ki}=\I{\psuper{mkj}{k}<i}$ for $j\in[n]$, $k\in[H+1]$, $i\in[H+1]$, and further constrain this by another indicator $c^j_{ki}$ (defined in \cref{app:alg-params}) that requires the data-point's least-squares uncertainty term to be sufficiently low, and the prediction non-negative (the contribution of the rest of the data will be analyzed separately).
Next, we define the least-squares solution for estimating the matrix-valued $F$, as well as the empirical average prediction and realization of $F$ on the data collected in the $m^\text{th}$ round.
For any $i\in[2:H]$, $k\in[i-1]$ (recall that $\tp$ denotes the tensor product):
\begin{align}\begin{split}
\hat\theta^{ti}_{\hat G \bar\theta} &= X_{mt}^{-1} \sum_{lkj\in \bI^m(t)} \phi^{lkj}_t \tp F_{\hat G \bar\theta}(S^{lkj}_i,\dots,\Sr^{lkj}) \quad\quad\text{for } t\in[i-1]
\\
y^{ki}_{\hat G \bar\theta} &= \frac1n\sum_{j\in[n]}c^j_{ki}{\phi^{mkj}_{\pshort{k}}}^\top \hat\theta^{\psuper{mkj}{k},i}_{\hat G \bar\theta} %
\quad\quad\quad\quad
\hat F^{ki}_{\hat G \bar\theta} = \frac1n\sum_{j\in[n]}c^j_{ki} F_{\hat G \bar\theta}(S^{mkj}_i,\dots,\Sr^{mkj}) \label{eq:hatF-def}
\end{split}
\end{align}

\vspace{-2mm}In \cref{sec:cons}, it is established via the usual least-squares analysis techniques and covering arguments, that with high probability the norm of the product of the matrix $y^{ki}_{\hat G \bar\theta}-\hat F^{ki}_{\hat G \bar\theta}$ and the projection matrix $\Proj_{\pa(\Q,i)}$ is small (\cref{lem:lse-measure-good,lem:avg-measure-good}). The next optimization problem tests if this is true in arbitrary directions:
\begin{optproblem}[Consistency check]\label{def:consistency-opt}
Input: $(\hat G,\bar\theta)$
\begin{align*}
\argmax
_{k\in[H-1],\,i\in[k+1:H],\,v\in \R^d: \norm{v}_2=1} 
&
v^\top \left(y^{ki}_{\hat G \bar\theta}-\hat F^{ki}_{\hat G \bar\theta}\right)v \\
\end{align*}
\end{optproblem}
\vspace{-5mm}\cref{lem:v-can-be-projd} shows that the projection $w=\Proj_{Z(\Q,i)} v$ is close to $v$,
where $v$ is the outcome of \cref{def:consistency-opt}.
Also, Lemmas~\ref{lem:v-can-be-projd}--\ref{lem:avg-measure-good} imply that if the consistency check fails (i.e., Line~\ref{line:decoydetect} is executed because the value of \cref{def:consistency-opt} is large), then
$w$ aligns well with the 
subspace $\Proj_{\perp (\Q,i)}$ projects to,
and therefore $\Q$ stays a valid preconditioning after appending $w$ to the list of values $\Q$ is calculated from (\cref{lem:decoy-find}). Thus, $\Q$ is always a valid preconditioning.

\section{Proof overview}\label{sec:proof-overview}

The proof of \cref{thm:main} is presented in \cref{app:proof-main}. It is composed of the following main steps:
First, we bound
the number of times the consistency check can fail
(i.e., Line~\ref{line:decoydetect} is executed)
by \cref{lem:valid-precond-d1-limit}.
Combining this with \cref{lem:m-max}, an elliptical potential argument bounding
the number of times the average uncertainty can be large (these are the only two ways that the main iteration can continue)
implies a sample-complexity result for \qpieleanor (\cref{cor:mp-max}).
Having limited the number of times the consistency check can fail, we derive guarantees regarding the performance of the policy returned by the algorithm:
Via an induction argument (\cref{lem:d-induction}) we show \cref{cor:perf-guarantee}, which shows that with high probability
the difference between the optimization value of \cref{def:opt-problem}, $C_{\hat G,\bar\theta}(s_1)$ and $v^{\pi^{mH}}$ scales
with the average \vspace{-1mm} uncertainty term $\sum_{i=1}^H\bar\sigma_k^m$.
Thus, they are close when \qpieleanor returns in Line~\ref{line:return}.
This is complemented with the \emph{optimism} property proved in \cref{lem:optimism}, stating that the optimization value $C_{\hat G,\bar\theta}(s_1)$ is close to $v^\star(s_1)$.
Combined, this proves \cref{thm:main}.

\section{Future work}\label{sec:future}

Since we are not aware of a computationally efficient implementation of \qpieleanor, it remains an open question whether the problem of learning near-optimal policies from online interactions with a $q^\pi$-realizable MDP (\cref{prob:main}) is possible if the computational resources as well as the sample complexity are bounded by a polynomial in the relevant parameters.
One approach is to replace \textsc{Eleanor} with LSVI-UCB as the underlying algorithm, as the latter, despite having worse sample complexity, has a computationally efficient implementation \citep{jin2020provably}.
The challenge is to compute the optimal solution for the parameter $\hat G$ in \cref{def:opt-problem}.
This parameter interacts with the least-squares targets in a highly nonlinear way.
We have been unable to derive a computationally efficient approximation that has an additive instead of a multiplicative approximation error (additive errors increase linearly in $H$, while multiplicative errors increase exponentially).
Alternatively, it may be possible to show a computational hardness result for \cref{prob:main} by e.g., reducing it to the satisfiability problem.
These are left for future work.
Our work on the realizability of auxiliary functions (\cref{sec:auxiliary-real}) may be of independent interest for designing provably efficient algorithms for related problem settings, e.g., the setting of $q^\pi$-realizability in batch RL, where the data collection is not controlled.

\bibliography{linear_fa}
\newpage
\appendix

\section{Notation}\label{app:not}

As usual, we use $\R$, $\N$, and $\N^+$ to denote the set of reals, non-negative and positive integers, respectively.
For $i\in\N^+$, let $[i]=\{1, \dots, i\}$; for another positive integer $j$, let $[i:j]=\{i,\dots,j\}$ if $i\le j$, and $[i:j]=\{\}$ otherwise. For $a,b,x\in\R$, let $\clip_{[a,b]}(x)=\min\{\max\{x,a\},b\}$ and let $\ceil{x}$ denote the smallest integer i such that $i\ge x$.
Let $\zero$ be the all-$0$ vector in $\R^d$ and $I$ the $d$-dimensional identity matrix.
For a (square) matrix $V$, let $V^\dag$ denote its Moore-Penrose inverse, and $\trace(V)$ denote its trace.
Let $\pd$ (and $\psd$) denote the set of positive definite (and positive semi-definite, respectively) matrices in $\R^{d\times d}$. For some $A\in\psd$ let $A^\frac12$ denote the unique matrix $B\in\psd$ such that $A=BB$.
For $V \in \pd$ and $x\in\R^d$, let $\norm{x}_V^2=x^\top Gx$.
For matrices $A$ and $B$, we say that $A \mge B$ (or $A \mle B$) if $B-A$ (or $A-B$, respectively) is positive semidefinite.
$\kernel(A)$ and $\image(A)$ are the kernel (or null space), and image, respectively, of matrix $A$.
For compatible vectors $x,y$, let $\ip{x,y}$ be their inner product: $\ip{x,y}=x^\top y$.
We write $y\tp A$ for the tensor product between $y$ and matrix $A$, and then $\ip{x,y\tp A}=\ip{x,y}A$. %
Where $\Q$ and $h$ are obvious from the context, we write $v_\pa$ and $v_\perp$
for $v_{\pa(\Q,h)}$ and $v_{\perp(\Q,h)}$, respectively. 
Throughout the paper, we omit commas between quantities in subscripts or superscript for clarity of presentation, for example, by writing $A_{bc}$ for $A_{b,c}$.

For the big-Oh notation $\ordo$, we introduce its counterpart $\ordot$ that hides logarithmic factors of the problem parameters $(d,H,\epsilon^{-1},\zeta^{-1}, \featurebound,\thetabound)$.

\section{Parameters of \cref{alg:main}}\label{app:alg-params}

\begin{align}%
n &= \ordot\left(d^5 H^6 / \epsilon^2 \right) &\text{(for precise value see \cref{eq:n-def})} \nonumber\\
\omega &= 7\dfoursq +7/3 =\ordot(d) \label{eq:omega-def}\\
\gamma^{-1} &=8d =\ordot(d) \label{eq:gamma-def}\\
\beta &= \ordot(H^{1.5}d) & \text{(for precise value see \cref{eq:beta-choice})} \label{eq:beta-approx}\\
\alpha^{-1} &= \frac{\sqrt{2d}\dfour H^2}{\sqrt{\gamma}\epsilon}=\ordot(d^{1.5}H^2/\epsilon)\label{eq:alpha-def}\\
\lambda^{-1}&=(4d_0\thetabound/\alpha)^{2} %
\\
\mmax &= \beta^2 \log\left(1+\frac{Hmn\featurebound^2}{d\lambda}\right)+1 = \ordot\left(H^3d^2\right) \nonumber\\
\mpmax&=\mmax+Hd_1=\ordot(H^3d^2) \nonumber
\end{align}
\begin{align}
\label{eq:barsigma-def}
\bar\sigma^m_k&=\frac1n\sum_{j\in[n]} \tilde c^j_{k,H+1} \min\left\{2(\beta\omega dH)^{-1}, \norm{\phi^{mkj}_{\pshort{k}}}_{X_{m,\psuper{mkj}{k}}^{-1}} \right\} & \text{for $k\in[H]$} \\
\end{align}
\begin{align}
\label{eq:c-indicator-def}
c^j_{ki}&=\I{\psuper{mkj}{k}<i \text{ and } \norm{\phi^{mkj}_{\pshort{k}}}_{X_{m,\psuper{mkj}{k}}^{-1}} < 2(\beta\omega dH)^{-1}\text{ and } \ip{\phi^{mkj}_{\pshort{k}},\bar\theta_{\psuper{mkj}{k}}}\ge0}&
\end{align}
\begin{align*}
\text{average\_uncertainty} &= \sum_{k=1}^H \bar\sigma_k^m\nonumber \\
\text{uncertainty\_threshold} &= \epsilon/(dH^2\beta\omega)\nonumber\\
\text{discrepancy\_threshold} &= \bar\sigma_k^m\beta\omega+3\frac{\epsilon}{dH^2} \nonumber
\end{align*}
Assumption on the maximum discrepancy:
\begin{align}\label{eq:eta-small-assumption}
\eta &\le \frac\alpha{10 d_0} \min\left\{\epsilon/(dH^3\omega),1/\sqrt{\mpmax nH}\right\} = \ordot\left( \frac{\epsilon^2}{ d^6 H^8}\right)
\end{align}
\section{Proof of \cref{prop:norange-linear-mdp}}

\begin{proof}[\textbf{Proof of \cref{prop:norange-linear-mdp}, and the MDP conversion argument}]\label{proof:norange-linear-mdp}\label{app:skipconvert}
First, for {\em (i)}, we show the linearity of rewards with $\theta_1,\dots,\theta_H$. 
For this
take any $h\in[H]$.
Fix any policy $\pi\in\Pi$ and let $\bar\theta_h\in\B(\thetabound)$ be such that
for all $(s,a)\in\cS_h\times[\cA]$,
$q^\pi(s,a)\approx_\eta \ip{\phi(s,a),\bar\theta_h}$ 
(the existence of such a $\bar\theta$ follows from \cref{def:q-pi-realizable}).
If $h=H$, $\E_{R~\sim\cR(s,a)} [R] = q^\pi(s,a)$, so $\theta_H=\bar\theta_H$ satisfies \cref{def:lin-mdp}.
For $h<H$, let $f:\cS_{h+1}\to\R$ be defined as $f(s)=v^\pi(s)$.
Fix an arbitrary $\Q\in\pdseq$, e.g., $\Q=(I,\dots,I)$.
Since $v^\pi(s) \in [0,H]$ and $\range_Q(s) \ge \range(s)/\sqrt{2d} \ge \alpha/\sqrt{2d}$ by \cref{prop:opt-design-gap-and-vstar-gap}, $f$ is $\alpha/(\sqrt{2d}H)$-admissible, and therefore by \cref{lem:admissible-realizability} we can take $\tilde\theta_h\in\B(4Hd_0\sqrt{2d}\thetabound/\alpha)$
such that for all $(s,a)\in\cS_h\times[\cA]$, 
\[\E(v^\pi(S_{h+1}) \,|\, s,a) \approx_{\sqrt{2d}H\eta_0} \ip{\phi(s,a),\tilde\theta_h}\,,\]
where, as before, $\eta_0=5d_0\eta/\alpha$. Since
\[
\E_{R~\sim\cR(s,a)} (R) = q^\pi(s,a)- \E(v^\pi(S_{h+1}) \,|\, s,a)\,,
\]
letting $\theta_h=\bar\theta_h-\tilde\theta_h$
satisfies {\em (i)} of \cref{def:lin-mdp}
with $\kappa=\eta+\sqrt{2d}H \eta_0 = \eta+ 5 H \sqrt{2d} d_0 \eta/\alpha$.

To show {\em (ii)}, take any 
$f:\cS \to [0,H]$ and $h\in[H-1]$.
As before, $f$ is $\alpha/(\sqrt{2d}H)$-admissible, therefore
\cref{lem:admissible-realizability} immediately provides $\theta'_h$ satisfying the required conditions.

Therefore, the MDP is shown to be linear with misspecification $\eta+\sqrt{2d}H\eta_0$, and parameter bound $\thetabound(4Hd_0\sqrt{2d}/\alpha+1)$.
\end{proof}

\textbf{\textsl{Sketch of the $q^\pi$-to-linear MDP conversion argument.}}
We elaborate on the conversion  to linear MDP mechanism presented in \cref{sec:qpi-to-linear}.
As the basis of this argument is that an idealistic range-determining oracle is present, we note that this argument only serves as intuition and is otherwise tangential to our proof. 
Instead of a direct approach of learning this oracle, our proof argues that learning about this oracle happens whenever there is a need (performance shortfall) for it.
A formal reduction to linear MDPs given this oracle however is fairly straightforward but cumbersome, with the caveat that the linear MDP will end up with $dH$ (instead of $d$) dimensional features. %
One would proceed by copying the features of each state $s$ in stage $h$ into the $h^\text{th}$ chunk of size $d$ of this vector of size $dH$ (the rest of the vector remains zero). A similar transformation is applied to all $\theta_h(\pi)$. Then, $H$ copies are made of each high-enough-range state, with all possible stages (but keeping the feature vectors). These will be the states of the new MDP we construct. When a transition from state $s$ leads to skipped states, the linear MDP returns with the copy of the first non-skipped state that has a stage counter of $\stage(s)+1$, so that in this linear MDP the stage numbers are consecutive (as required by our definitions). $q^\pi$-realizability of this modified MDP is easy to show, and (as it has no low-range states) \cref{prop:norange-linear-mdp} can be used to show that the modified MDP is linear. To account for the fact that this new MDP may finish an episode in fewer than $H$ steps due to the skips, we add a special, zero-reward, self-transitioning state called ``episode-over''. To ensure that the MDP stays linear, we extend the feature vectors of each state by a scalar $1$, and a scalar indicator of being in this state, with all original features of the ``episode-over'' state defined to be zero. It is easy to see that this construction leads to a linear MDP with the desired action-value functions.

\section{Intuition behind our method and proof strategy from the perspective of \textsc{Eleanor} \citep{zanette2020learning}}\label{app:eleanor-perspective}

The starting point of our method is the \textsc{Eleanor} algorithm, which is designed for linear MDPs.
Similarly to \qpieleanor, \textsc{Eleanor} solves an optimistic optimization problem inside a loop.
The optimization problem computes optimistic estimates  $\bar\theta_t$ of the parameters of the MDP simultaneously for all $t\in[H]$, and in each iteration of the loop, more data is collected according to the policy that is optimal for the MDP defined by the estimated parameters.
Initial estimates $\hat\theta_t$ are computed via solving least-squares problems whose 
covariates are the features corresponding to state-action pairs $(S_t, A_t)$ from all the data collected so far, while the corresponding least-squares targets are computed as the sum of the immediate reward $R_t$ and the estimated value for $S_{t+1}$, computed from $\bar{\theta}_{t+1}$.
$\bar\theta_t$ is then optimistically chosen 
as the solution of the optimization problem,  
in the neighborhood (confidence ellipsoid) of $\hat\theta_t$, the solution to this least-squares problem.
It is shown that this optimistic choice of estimates results in an optimistic estimate of the value of $v^\star$ of the initial state, and the regret is upper bounded in terms of the sum of elliptic potentials of the covariates.

This argument appears in our analysis too, with minor modifications due to our PAC-like setting (instead of aiming to bound the regret), leading to our final-iteration condition of Line~\ref{line:cons-check-pass} in \cref{alg:main}.
Our \cref{def:opt-problem} is similar to \textsc{Eleanor}'s, and the parameters $\bar\theta_t$ and $\hat\theta_t$ have the same meaning.
A key difference between the optimization problems of \textsc{Eleanor} and \qpieleanor are how the least-squares targets are determined. For \textsc{Eleanor}, it is the sum of the immediate reward $R_t$ and its estimated value for $S_{t+1}$);
with this target, only one on-policy rollout is required for each episode in order to get the least-squares parameter estimate for all $H$ stages. 
In contrast, our least-squares targets are formed as the sum of $R_t + \ldots + R_{t+i}$ and the estimated value for $S_{t+i+1}$, where $i$, the number of stages ``skipped'', 
depends on the guess $\hat G$. The guess $\hat G$ is selected only in \cref{def:opt-problem}, and we do not know its value at the time of data collection, so we cannot know which stages will have to be skipped for each rollout. Therefore, (i) we need access to the rewards of the current policy at any stage (similarly to \textsc{Eleanor}), 
and hence we run the current policy to any stage (including the last one); and (ii) perform rollouts with the fixed policy $\pi^0$ (from any stage) to be able to estimate the reward $R_t+\ldots+R_{t+i}$ collected while skipping over $i$ stages (for any $i$). To ensure this happens for every stage, we start Phase II from every stage $k$, resulting in the additional for loop in Line~\ref{line:k-for} of \cref{alg:main} compared to \textsc{Eleanor}. Finally, the randomization in Phase I is applied to make the optimization problem smooth, as described in \cref{sec:alg}.

One could analyze this algorithm similarly to the analysis of \textsc{Eleanor} if it were not for the fact that the least-squares targets we just introduced are not realizable in general. We can, however, prove the realizability of certain components of the matrix-valued version of these targets, $F$ (\cref{lem:dual-admissibility} and \cref{cor:F-components-realizable}). 
This enables us to detect when the realizability of our least-squares targets fail, measure the direction (component) of the largest error, and learn from that. This is the job of \cref{def:consistency-opt}: 
$\hat{F}^{ki}_{\hat{G}\bar{\theta}}$
corresponds to the matrix-valued empirical measurements of 
$F$, while the
$y^{ki}_{\hat{G}\bar{\theta}}$ 
are the average predictions of the same quantities. If the targets are realizable, which happens if we manage to skip the right number of stages), these matrices are very close; if not, the direction of their largest discrepancy tells us something about $\perp\!(Q, i)$, and allows us to learn.

Optimism ties all this together: either there is no shortfall between predicted and measured $q$-values (and we are done) or we grow the elliptical potential of $X$ (the two cases present in the analysis of \textsc{Eleanor}, \citet{zanette2020learning}), or we grow the elliptical potential of $\Q$ (the new case due to the lack of realizability guarantees).

\section{Proof of \cref{thm:main}}\label{app:proof-main}

In this section we present the proof of \cref{thm:main}. Recall that some quantities are defined in \cref{app:alg-params}.

\subsection{Checking consistency} \label{sec:cons}

We introduce some lemmas to establish the required guarantees of the consistency checker. Their proofs, which rely on the usual least squares analysis techniques and covering arguments, are presented in \cref{proof:v-can-be-projd}.
\newcommand{\bM}{\mathbf{M}}
\begin{lemma}\label{lem:v-can-be-projd}
Let $(k,i,v)$ be the outcome of \cref{def:consistency-opt} any time during the execution of \qpieleanor,
and let $w=\Proj_{Z(\Q,i)} v$ as in the algorithm.
Then,
\[
w^\top \left(y^{ki}_{\hat G \bar\theta}-\hat F^{ki}_{\hat G \bar\theta}\right)w \ge v^\top \left(y^{ki}_{\hat G \bar\theta}-\hat F^{ki}_{\hat G \bar\theta}\right)v - \frac{\epsilon}{dH^2\omega}
\]
\end{lemma}

\begin{lemma}\label{lem:lse-measure-good}
There is an event $\event{1}$ that happens
with probability at least $1-\zeta$, %
such that under $\event{1}$,
during the execution of \qpieleanor,
when the beginning of any iteration (Line~\ref{line:opt}) is executed,
for any $t\in[H-1]$, $i\in[t+1:H]$,
for any $\hat G\in\bG$, $\bar\theta\in\bbt$,
 and $v,w\in\B(1)$,
 for all $(s,a)\in\cS_t\times[\cA]$,
\begin{align}
&\abs{v_\pa^\top\left({\phi(s,a)}^\top \hat\theta^{ti}_{\hat G \bar\theta} -
\E_{\pi^0,s,a} \bar F_{\hat G \bar\theta}(S_i) \right)w }
 \le \norm{\phi(s,a)}_{X_{mt}^{-1}}\beta + \frac{\epsilon}{dH^2\omega}\nonumber\,,
\end{align}
where $\bullet_\pa$ denotes $\bullet_{\pa(\Q,i)}$.
\end{lemma}

The next lemma uses the average least-squares predictions' (capped) uncertainty term $\bar\sigma^m_k$ (defined in \cref{eq:barsigma-def}),
where the average is taken over predictions from the state-action pair where Phase I of $\skippypolicy(\cdot,\cdot,k)$ ends.

\begin{lemma}\label{lem:avg-measure-good} %
There is an event $\event{2}$
with probability at least $1-\zeta$, %
such that under $\event{1}\cap\event{2}$,
during the execution of \qpieleanor,
when \cref{def:consistency-opt} is solved (Line~\ref{line:consopt}),
for $(\hat G,\bar\theta)$ as recorded in Line~\ref{line:opt} %
for all $k\in[H-1]$, $i\in[k+1:H]$,
and $v,w\in\B(1)$,
\begin{align*}
\abs{
v_\pa^\top \left(y^{ki}_{\hat G \bar\theta}-\hat F^{ki}_{\hat G \bar\theta}\right) w} &\le \bar\sigma_k^m\beta + 3\frac{\epsilon}{dH^2\omega}\,
\end{align*}
where $\bullet_\pa$ denotes $\bullet_{\pa(\Q,i)}$.
\end{lemma}

Together, these lemmas can be used to show that the vector $w$ derived from Line~\ref{line:consopt} in \qpieleanor is sufficiently aligned with both $Z(\Q,\cdot)$ and 
the subspace $\Proj_{\perp(\Q,\cdot)}$ projects to, 
which leads to the following important result:
\begin{lemma}\label{lem:decoy-find}
Under the $\event{1}\cap\event{2}$,
if Line~\ref{line:decoydetect} is executed %
any time during the execution of \qpieleanor (i.e., when the consistency check fails), then the resulting $\Q$ continues to be a valid preconditioning.
\end{lemma}

From now on, our lemmas assume the high-probability events of \cref{lem:lse-measure-good,lem:avg-measure-good} hold, and therefore $\Q$ is a valid preconditioning at any time during the execution by \cref{lem:decoy-find}.%

\subsection{Sample complexity bounds}\label{sec:sampcompbound}
We bound the number of iterations of $m$ that \qpieleanor can execute.
The proof of the following lemma is presented in \cref{proof:m-max}:
\begin{lemma}\label{lem:m-max}
Throughout the execution of \qpieleanor,
$m\le \mmax$.
\end{lemma}
Note that throughout the execution of \qpieleanor,
$m'\le \mpmax$.
As $m'-m$ equals the number of times Line~\ref{line:decoydetect} is executed, i.e., the sum of sequence lengths corresponding to $\Q$, by \cref{lem:valid-precond-d1-limit}, 
\begin{corollary}\label{cor:mp-max}
Under $\event{1}\cap\event{2}$,
\qpieleanor returns with a policy before exiting the while loop of Line~\ref{line:while}, and
as each iteration executes $Hn$ trajectories in Line~\ref{line:execute-traj}, the number of interactions of \qpieleanor with the MDP is bounded by $\ordot\left(H^{11}d^7/\epsilon^2\right)$.
\end{corollary}

\subsection{Performance guarantee}\label{sec:perf}

We next consider the $m^\text{th}$ iteration of \qpieleanor under the assumption that
the consistency check passes, that is, Line~\ref{line:cons-check-pass} is executed.
We intend to guarantee the performance of $\pi^{mH}$ in terms of $\sum_{t=1}^H \bar\sigma_k^m$, given that the optimization value $x$ satisfies $x\le \bar\sigma_k^m\beta\omega+3\frac{\epsilon}{dH^2}$ (which follows from the execution reaching Line~\ref{line:cons-check-pass}).
Next we introduce variants of $c^j_{ki}$ and $\tilde c^j_{ki}$ (\cref{eq:c-indicator-def}) which act, instead of the data collected during the execution of the algorithm, on a trajectory $(S_h,A_h,R_h)_{h \in [H]}$ and corresponding stage mapping $p$ obtained by an independent run of \skippypolicy, which will be clear from the context:
$\tilde c_{ki}=\I{p(k)<i}$, and
\begin{align*}%
c_{ki}=\I{p(k)<i \text{ and } \norm{\phi(S_{p(k)},A_{p(k)})}_{X_{m,p(k)}^{-1}} < 2(\beta\omega dH)^{-1}\text{ and } \ip{\phi(S_{p(k)},A_{p(k)}),\bar\theta_{p(k)}}\ge0}\,.
\end{align*}

\begin{remark}\label{rem:pol-same}
In our analysis we rely on the obvious fact that the laws of the trajectories of $\skippypolicy(\hat{G},\bar\theta,k)$ and $\skippypolicy(\hat{G},\bar\theta,k+1)$
are the same until stage $p(k+1)$ (as the policies are the same until then), for any parameters $\hat{G}$ and $\bar{\theta}$. This includes $S_{p(k+1)}$
but not $A_{p(k+1)}$ if $p(k+1)\le H$,
and includes the whole trajectory ending with $\Sr$ otherwise.
\end{remark}

We prove the following using induction on $k=H,\dots,1$ in \cref{proof:d-induction}:
\begin{lemma}\label{lem:d-induction}
There is an event $\event{3}$
with probability at least $1-3\zeta$,
such that under $\event{1}\cap\event{2}\cap\event{3}$,
during the execution of \qpieleanor,
whenever Line~\ref{line:cons-check-pass} is executed,
for $(\hat G,\bar\theta)$ as recorded in Line~\ref{line:opt} of the current iteration,
for $k\in[H]$,
\begin{align}\label{eq:d-induction}
\bar C^k:=\E_{\pi^{mk},s_1} \tilde c_{k,H+1} %
C_{\hat G \bar\theta}(S_{p(k)}) \le
\E_{\pi^{mH},s_1} \sum_{u=p(k)}^{H}R_u + 2\sum_{i=k}^H\bar\sigma_k^m\beta\omega dH+4(H-k+1)\frac{\epsilon}{H}\,.
\end{align}
\end{lemma}
As $S_1=s_1$ is fixed and $\tau(s_1)=1$, we get the following corollary, which shows that the value $C_{\hat{G}\bar{\theta}}$ of the solution $(\hat{G},\bar{\theta})$ of \cref{def:opt-problem} can be used as a lower bound on the value of the policy $\pi^{mH}$ up to the uncertainty and some $\epsilon$ terms:
\begin{corollary}\label{cor:perf-guarantee}
Under $\event{1}\cap\event{2}\cap\event{3}$,
the value of \cref{def:opt-problem} with the solution $(\hat{G},\bar{\theta})$ satisfies
\begin{align*}%
C_{\hat G \bar\theta}(s_1)=\bar C^1 \le \E_{\pi^{mH},s_1} \sum_{u=1}^{H}R_u + 2\sum_{i=1}^H\bar\sigma_k^m\beta\omega dH^2+4\epsilon = v^{\pi^{mH}}(s_1) + 2\sum_{i=1}^H\bar\sigma_k^m\beta\omega dH^2+4\epsilon\,.
\end{align*}
\end{corollary}

\subsection{Optimism of \cref{def:opt-problem}}\label{sec:optimism}

The following establishes the optimistic property, that is,
that the value of \cref{def:opt-problem} competes with $v^\star(s_1)$.
The proof relies on the fact that the correct guess $\hat G$
and a good choice of $\bar\theta$ %
are feasible
for the optimization problem,
combined with the fact that this $\bar\theta$ induces a policy $\pi=\skippypolicy(\hat G,\bar\theta,H)$
that takes action-value maximizing actions according to a very accurate approximation of action-values almost everywhere.
In fact, it only skips states whose range is at most $\epsilon/H$.
The proof is presented in \cref{proof:optimism}.

\begin{lemma}\label{lem:optimism}
There is an event $\event{4}$ 
with probability at least $1-\zeta$,
such that under $\event{1}\cap\event{2}\cap\event{4}$,
throughout the execution of \qpieleanor,
the value of \cref{def:opt-problem}
is at least $v^\star(s_1)-2\epsilon$.
\end{lemma}

\begin{proof}[Proof of \cref{thm:main}]
We combine \cref{lem:optimism} with \cref{cor:perf-guarantee}, \cref{cor:mp-max},
and the fact that the condition of Line~\ref{line:cons-check-pass} is satisfied when \qpieleanor returns with a policy,
to get that under %
$\event{1}\cap\event{2}\cap\event{3}\cap\event{4}$,
that is, with probability at least $1-6\zeta$,
\qpieleanor interacts with the MDP for at most $\ordot\left(H^{11}d^7/\epsilon^2\right)$ many steps,
before returning with the policy $\pi^{mH}$ that satisfies
\[
v^\star(s_1) \le
C_{\hat G \bar\theta}(s_1) +2\epsilon\le v^{\pi^{mH}}(s_1) + 2\sum_{i=1}^H\bar\sigma_k^m\beta\omega dH^2+6\epsilon \le v^{\pi^{mH}}(s_1)+8\epsilon\,,
\]
where the final inequality follows from the fact that when \qpieleanor returns in Line~\ref{line:return},
$\sum_{k=1}^H \bar\sigma_k^m \le\epsilon/(\beta\omega dH^2)$.
By scaling the parameters, this finishes the proof of \cref{thm:main}.
\end{proof}

\section{Deferred definitions and proofs for \cref{sec:range}}

\begin{proof}[\textbf{Proof of \cref{lem:theta-precond-norm-bound}}]\label{proof:theta-precond-norm-bound}
For any $\theta\in\Theta_h^\Q$, it holds that $\theta=\Q^{-1}_h \hat \theta$ for some $\hat\theta\in\Theta_h$.
Since $\norm{\hat\theta}_2\le\thetabound$, and writing $\Q_h$ as in \cref{def:valid-preconditions},
\begin{align*}
\textstyle
    \norm{\theta}_2^2={\hat\theta}^\top \left(\miota I
+\sum_{v\in C_h} vv^\top
\right) \hat\theta
\le \miota\thetabound^2+|C_h| %
\le 1+d_1\,,
\end{align*}
where we used \cref{def:valid-preconditions} and \cref{lem:valid-precond-d1-limit}.
Finally, we conclude that
$    \norm{\theta}_2 \le \sqrt{d_1+1}$.
\end{proof}

\begin{definition}\label{def:nearopt}
$(G_h^\Q, \rho_h^\Q)$ is a near-optimal design for $\Theta_h^\Q$, if
for any $\theta\in\Theta_h^\Q$,
\begin{align}
&\ip{v,\theta} = 0 \quad\text{for all } v\in\kernel(V(G_h^\Q, \rho_h^\Q)), \text{ and}\label{eq:v-max-rank}\\
&\norm{\theta}_{V(G_h^\Q, \rho_h^\Q)^\dag}^2\le 2d, \label{eq:near-opt-2d}\\
&\text{where } V(G_h^\Q, \rho_h^\Q)=\sum_{\pi\in G_h^\Q} \rho_h^\Q(\pi) (\theta_h^\Q(\pi)){(\theta_h^\Q(\pi))}^\top\,.\label{eq:design-matrix}
\end{align}
\end{definition}
An important corollary of the above definition is that if $M=\Proj_{\image(V(G_h^\Q, \rho_h^\Q))}$, then
${V(G_h^\Q, \rho_h^\Q)^\dag}^\frac12 V(G_h^\Q, \rho_h^\Q)^\frac12 M v = M v$,
and $\ip{\theta, Mv}=\ip{\theta,v}$ due to \cref{eq:v-max-rank}, and so
\begin{align}
	\theta^\top v = \theta^\top{V(G_h^\Q, \rho_h^\Q)^\dag}^\frac12 V(G_h^\Q, \rho_h^\Q)^\frac12 v \quad\text{ for all }\theta\in\Theta_h^\Q\text{ and } v\in\R^d.  \label{eq:pseudoinv-cool-for-cauchy}
\end{align}

\begin{proof}[\textbf{Proof of \cref{prop:opt-design-gap-and-vstar-gap}}]\label{proof:opt-design-gap-and-vstar-gap}
Take any $h\in[H]$,
$s\in\cS_h$, and $\Q\in\pdseq$. %
Take $i,j\in[\cA]$ %
such that $\range(s)=\sup_{\theta\in\Theta_h}\ip{\phio(s,i,j), \theta}$.
Then, %
\begin{align*}
\range(s)^2&= \sup_{\theta\in\Theta_h} \ip{\phio(s,i,j),\theta}^2
=\sup_{\theta\in\Theta_h^\Q}\ip{\phi_{\Q}(s,i,j),\theta}^2\\
&\le \sup_{\theta\in\Theta_h^\Q}\norm{\theta}^2_{V(G_h^\Q, \rho_h^\Q)^\dag} \norm{\phi_{\Q}(s,i,j)}^2_{V(G_h^\Q, \rho_h^\Q)}  \\
&\le 2d \phi_{\Q}(s,i,j)^\top \left(\sum_{\pi\in G_h^\Q}\rho_h^\Q(\pi)(\theta_h^\Q(\pi)){(\theta_h^\Q(\pi))}^\top\right) \phi_{\Q}(s,i,j)\\
&=2d \phi(s,i,j)^\top \left(\sum_{\pi\in G_h^\Q}\rho_h^\Q(\pi)(\theta_h(\pi)){(\theta_h(\pi))}^\top\right) \phi(s,i,j)\\
&\le 2d \max_{\pi\in G_h^\Q} \ip{\phio(s,i,j)^\top, \theta_h(\pi)}^2
\le 2d \range_\Q(s)^2\,,
\end{align*}
where the first inequality uses \cref{eq:pseudoinv-cool-for-cauchy} and the Cauchy-Schwarz inequality, and the second inequality follows by substituting the definition of $V(G_h^\Q,\rho_h^\Q)$ and using \cref{eq:near-opt-2d}.
Finally, the first inequality in the last line holds as we replace the weighted sum from the previous line with the maximum operator. %
We therefore get that %
$\range(s)\le \sqrt{2d}\range_\Q(s)$,
finishing the proof.
\end{proof}

\begin{proof}[\textbf{Proof of \cref{lem:valid-precond-d1-limit}}]\label{proof:valid-precond-d1-limit}
Take any $h\in[H]$ and the sequence $C_h$ corresponding to $\Q$.
Assume that this sequence is of length $l$,
and let $\Sigma_{h,i}= \miota I + \sum_{j=1}^{i}v_j v_j^\top$ for $i\in[l]$. %
By the second part of \cref{eq:precond-lownorm},
\[
l =
\sum_{i=1}^l \min\left\{ 1, 2\norm{\left(\miota I + \sum_{j=1}^{i-1}v_j v_j^\top\right)^{-\frac12} v_i}_2^2 \right\}
 \le 2\sum_{i=1}^l \min\left\{ 1, \norm{v_i}_{\Sigma_{h,i-1}^{-1}}^2 \right\}~.
\]
Applying the elliptical potential lemma (\cref{lem:elliptical-pot}),
\[
l \le 2\sum_{i=1}^l \min\left\{ 1, \norm{v_i}_{\Sigma_{h,i-1}^{-1}}^2 \right\}
\le 4d\log\left(\frac{\trace (\Sigma_{h,0}) + l \precondbound^2}{d\det(\Sigma_{h,0})^{1/d}}\right)
= 4d\log\left(1 + \frac{l \precondbound^2}{\miota d}\right)~.
\]
where $\Sigma_{h,0}=\miota I$ by definition.
Using that $\log(1+x)\le \sqrt{x}$ for $x\ge 0$, we have
$l\le 4d\sqrt{l\precondbound^2\thetabound^2/d}$, which implies
$l\le 16d\precondbound^2\thetabound^2$.
Substituting this into the previous bound yields
\[
l \le 4d\log\left(1 + \frac{(16d\precondbound^2\thetabound^2) \precondbound^2}{\miota d}\right)
= 4d \log\left(1 + 16\precondbound^4\thetabound^4\right)
=d_1 \,.\qedhere
\]
\end{proof}

\section{Deferred proofs for \cref{sec:auxiliary-real}}

For any vector $v\in\R^d$, we define $\overline v=v/\norm{v}_2$ as the unit vector in the direction of $v$ if $v\ne\zero$ and $\overline \zero=0$ otherwise.
For any $h\in[2:H]$, $s\in\cS_h$, the normalized version of the largest preconditioned feature difference is denoted by
\begin{align}\label{eq:phiwitness}
\bar{\phi}_\Q(s)=\overline{\phi_\Q(s,i,j)} \quad\text{ where } (i,j)=\argmax_{i',j'\in[\cA]}\norm{\phi_\Q(s,i',j')}_2\,.
\end{align}

\begin{proof}[\textbf{Proof of \cref{lem:admissible-realizability}}]
\label{proof:admissible-realizability}
Fix $h\in[H]$, $\alpha$-admissible $f:\cS_h\to\R$,
$t\in[h-1]$, and $\pi\in\Pi$.
Our aim is to construct policies $\pi^+_k, \pi^-_k \in\Pi$
for $k\in [d_0]$, such that for all $(s,a)\in\cS_t\times[\cA]$,
$\sum_{k\in[d_0]} (q^{\pi^+_k}(s,a)-q^{\pi^-_k}(s,a))$ is approximately proportional to the desired $\E_{\pi,s,a}f(S_h)$.
Let $G_{h,1}^\Q,G_{h,2}^\Q,\dots$ denote the policies in $G_h^\Q$ underlying the near-optimal design of $\Theta_h^\Q$, and
for any $s\in\cS_h$, denote by
$\ord(s)\in[d_0]$ the index of the policy maximizing the range of the action-value function in state $s$, that is, $G^Q_{h,\ord(s)}=\argmax_{\pi\in G_h^\Q} \max_{i,j\in[\cA]} (q^\pi(s,i)-q^\pi(s,j))$;
 \newcommand{\tildeG}[1]{{\tilde G(#1)}}
to simplify notation, we define $\tildeG{s}=G_{h,\ord(s)}^\Q$.
For $s\in\cS_h$ let
\begin{align*}
  (a^+(s),a^-(s))=
\begin{cases}
\argmax_{i,j\in[\cA]} \hat q^{\tildeG{s}}(s,i)-\hat q^{\tildeG{s}}(s,j) &\text{if $f(s)\ge 0$} \\
\argmin_{i,j\in[\cA]} \hat q^{\tildeG{s}}(s,i)-\hat q^{\tildeG{s}}(s,j) &\text{otherwise.} %
\end{cases}
\end{align*}
By \cref{eq:rangeq-def} and \cref{def:alpha-admissible} have that
\begin{align*}%
\abs{\hat q^{\tildeG{s}}(s,a^+(s))-\hat q^{\tildeG{s}}(s,a^-(s))} &= \range_\Q(s)
\ge \alpha |f(s)| \ge0 \,.%
\end{align*}

Since $q^{\tildeG{s}}(s,a^+(s))-q^{\tildeG{s}}(s,a^-(s)) \approx_{2\eta}
\hat q^{\tildeG{s}}(s,a^+(s))-\hat q^{\tildeG{s}}(s,a^-(s))$,
if $\alpha |f(s)|\ge 4\eta$, we have
\begin{align}\label{eq:range-gap-lb-22}
\begin{split}
q^{\tildeG{s}}(s,a^+(s))-q^{\tildeG{s}}(s,a^-(s))
\ge \alpha f(s)-2\eta\ge \frac{\alpha}{2} f(s) > 0
 &\quad\quad\text{if } f(s)\ge 0\\
q^{\tildeG{s}}(s,a^+(s))-q^{\tildeG{s}}(s,a^-(s))
\le \alpha f(s)+2\eta\le \frac{\alpha}{2}f(s)< 0 &\quad\quad\text{otherwise.}%
\end{split}
\end{align}
Let us define
$f':\cS_h\to\R$ as
\begin{align*}
  f'(s)=
\begin{cases}
	\frac{\alpha f(s)/2}{q^{\tildeG{s}}(s,a^+(s))-q^{\tildeG{s}}(s,a^-(s))}&\text{if }\alpha |f(s)|\ge 4\eta %
	\\
	0 &\text{otherwise.}%
\end{cases}\end{align*}
By \cref{eq:range-gap-lb-22}, there can be no division by zero in the above definition, and $0\le f'(s)\le 1$.

Now we are ready to define $\pi^+_k$ and $\pi^-_k$. Both policies follow $\pi$ up to stage $h-1$, when they switch to $G_{h,k}^\Q$, except if at stage $h$ a state $s \in \cS_h$ is such that $G_{h,k}^\Q$ has the maximal action-value function range. In this case $\pi^+_k$ selects $a^+(s)$ with probability $f'(s)$ and $a^-(s)$ with probability $1-f'(s)$, while $\pi^-_k$ always selects $a^-(s)$. Formally, for $k\in[d_0]$, we define for $s\in\cS$
\begin{align*}
  \pi^+_k(s)=
\begin{cases}
	\pi(s)&\text{if }\stage(s)< h; %
	\\
	a^+(s) \text{ w.p. } f'(s), \text{ and } a^-(s) \text{ w.p. } 1-f'(s)&\text{if } \stage(s)=h \text{ and } \ord(s)=k;
      \\
    G_{h,k}^\Q(s),
    &\text{otherwise,} %
\end{cases}\end{align*}
where w.p. stands for \emph{with probability}. Similarly,
\begin{align*}
  \pi^-_k(s)=
\begin{cases}
	\pi(s)&\text{if }\stage(s)< h; %
	\\
	a^-(s) \text{ w.p. } 1&\text{if } \stage(s)=h \text{ and } \ord(s)=k;
      \\
    G_{h,k}^\Q(s)
    &\text{otherwise.} %
\end{cases}\end{align*}
Note that $\pi^+_k\in\Pi$ and $\pi^-_k\in\Pi$, as desired.
Since for all $k\in[d_0]$, the policies follow $G_{h,k}$ for $s\in\cS_{t'}$ for $t'>h$,
therefore for all $k\in[d_0]$,
\begin{align}\label{eq:vs-same-later-new}
v^{\pi^-_k}(s)&=v^{\pi^+_k}(s)=v^{G_{h,k}^\Q}(s) \quad\text{for all $s\in\cS_{h+1}$, and}\\
q^{\pi^-_k}(s,a)&=q^{\pi^+_k}(s,a)=q^{G_{h,k}^\Q}(s,a) \quad\text{for all $(s,a)\in\cS_h\times[\cA]$.} \label{eq:qs-same-later-new}
\end{align} %

Also, %
for any $s\in\cS$ with $\stage(s)<h$ and any $a\in[\cA]$,
\begin{align*}
\begin{split}%
\sum_{k\in[d_0]} \left(q^{\pi^+_k}(s,a)-q^{\pi^-_k}(s,a)\right)
&= \E_{\pi,s,a} \sum_{k\in[d_0]} \left(v^{\pi^+_k}(S_h)-v^{\pi^-_k}(S_h)\right)\\
&= \E_{\pi,s,a}\left( v^{\pi^+_{\ord(S_h)}}(S_h)-v^{\pi^-_{\ord(S_h)}}(S_h) \right)\\
&= \E_{\pi,s,a}\Big( q^{\tildeG{S_h}}(S_h,a^+(S_h))f'(S_h)+
q^{\tildeG{S_h}}(S_h,a^-(S_h))(1-f'(S_h))\\
&\quad\quad\qquad-q^{\tildeG{S_h}}(S_h,a^-(S_h)) \Big)\\
&= \E_{\pi,s,a}\left(f'(S_h)\left( q^{\tildeG{S_h}}(S_h,a^+(S_h))-
q^{\tildeG{S_h}}(S_h,a^-(S_h))\right) \right) \\
&=\E_{\pi,s,a}\one{\alpha |f(S_h)|\ge 4\eta}\frac{\alpha}{2} f(S_h) \approx_{2\eta} \frac{\alpha}{2} \E_{\pi,s,a}f(S_h)\,,
\end{split}
\end{align*}
where %
the first line is due to both $q^{\pi^+_k}$ and $q^{\pi^-_k}$ following $\pi$ on states with stage less than $h$,
the second line follows from the fact that for any $s\in\cS_h$, $\pi^+_k(s)=\pi^-_k(s)$ for any $k\ne \ord(s)$; combining this with \cref{eq:vs-same-later-new} leads to all $k\ne \ord(s)$ terms of the sum to cancel.
The third line follows from expanding the definition of the policies and \cref{eq:qs-same-later-new}.

Let $\tilde \theta=\frac2\alpha \sum_{k\in[d_0]} \left(\theta_t(\pi^+_k)-\theta_t(\pi^-_k)\right)$.
Since $\norm{\theta_t(\cdot)}_2 \le \thetabound$ by definition, %
we have $\norm{\tilde \theta}_2\le 4d_0\thetabound/\alpha$.
By \cref{def:q-pi-realizable},
for all $(s,a)\in\cS_t\times[\cA]$,
\begin{align*}
\ip{\phi(s,a), \frac\alpha2 \tilde\theta} &\approx_{2d_0\eta} \sum_{k\in[d_0]} q^{\pi^+_k}(s,a)-q^{\pi^-_k}(s,a)
\approx_{2\eta} \frac{\alpha}{2} \E_{\pi,s,a}f(S_h),
\end{align*}
and hence
\begin{align*}
\ip{\phi(s,a), \tilde\theta} &\approx_{4(d_0+1)\eta/\alpha} \E_{\pi,s,a}f(S_h)\,.
\end{align*}
Since $4(d_0+1)\eta/\alpha \le \eta_0=5d_0\eta/\alpha$ as $d_0 \ge 4$ by definition, this completes the proof.
\end{proof}

\begin{proof}[\textbf{Proof of \cref{lem:dual-admissibility}}]\label{proof:dual-admissibility}
Take any $s\in\cS_h$.
For the correct guess, $\range_\Q^{\hat G}(s)=\range_\Q(s)$.
Then, using that $\norm{\bar{\phi}_\Q(\cdot)}_2\le1$,
$\trace(\bff(s))\le \range_\Q(s)\frac{\sqrt{2d}H^2}\epsilon$,
proving the second claim of the lemma (as $\gamma\le1$).

For the first claim, take any $\hat G=(\hat\vartheta^i_h)_{h\in[2:H],i\in[d_0]}\in\bG$.
Let $\phi'$ be the unnormalized version of $\bar{\phi}_\Q(s)$ of \cref{eq:phiwitness}, that is, $\phi'=\phi_\Q(s,i,j)$ for the same $i,j$ as in \cref{eq:phiwitness} (i.e., with the largest $\ell_2$-norm).
Then, using that $\hat G\in\bG$,
\begin{align*}
\range_\Q^{\hat G}(s) = 
\max_{k\in[d_0]}\max_{i,j} \ip{\phi_\Q(s,i,j),\hat\vartheta_h^k}
\le \norm{\phi'}_2 \max_{k\in[d_0]}\norm{\hat\vartheta_h^k}_2
\le \norm{\phi'}_2 \sqrt{d_1+1}.
\end{align*}

Using that above in combination with $|f(s)|\le H$, $v,w\in\B(1)$, $\norm{\bar{\phi}_\Q(s)}_2\le1$, we obtain
\begin{align*}
|v_\pa^\top\bff(s) w|
&\le \abs{\ip{\bar{\phi}_\Q(s), v_\pa}\ip{\bar{\phi}_\Q(s), w}} \range_\Q^{\hat G}(s)\frac{\sqrt{2d}H^2}\epsilon \\
&\le \norm{\bar{\phi}_\Q(s)_\pa}_2 \norm{\phi'}_2
\dfour  \frac{\sqrt{2d}H^2}\epsilon 
= \norm{\phi'_\pa}_2 \dfour  \frac{\sqrt{2d}H^2}\epsilon\,.
\end{align*}
As the eigenvalues of $V(G_h^\Q,\rho_h^\Q)=\sum_{\pi\in G_h^\Q} \rho_h^\Q(\pi) (\theta_h^\Q(\pi)){(\theta_h^\Q(\pi))}^\top$ corresponding to the subspace in which $\phi'_\pa$ lies are by definition at least $\gamma$, we can write
\begin{align*}
(\range_\Q(s))^2 \ge \max_{\pi\in G_h^\Q} \ip{\phi',\theta_h^\Q(\pi)}^2
\ge {\phi'}^\top V(G_h^\Q,\rho_h^\Q) \phi'
\ge {\phi'_\pa}^\top V(G_h^\Q,\rho_h^\Q) \phi'_\pa
\ge \norm{\phi'_\pa}_2^2 \gamma\,.
\end{align*}
Combining with the previous result, we get that 
\begin{align*}
\range_\Q(s)\ge \sqrt{\gamma}\norm{\phi'_\pa}_2 \ge \frac{\sqrt{\gamma}\epsilon}{\sqrt{2d}\dfour H^2} |v_\pa^\top\bff(s) w|
= \alpha |v_\pa^\top\bff(s) w|,
\end{align*}
finishing the proof.
\end{proof}

\section{Deferred proofs for \cref{sec:cons}}

The definitions (\cref{eq:fact:E-to-probab-form,eq:E-def}) immediately give rise to the following facts:
\begin{align}\begin{split}\label{eq:fact-bounds-E-D}
&D(\traj{i})\in\left[-\sum_{u=i}^{H} r_u,H\right]\subseteq[-H,H] \text{ and } \tau(\cdot)\in[0,1],\, \text{implying}\\
& E^\to(\traj{i})\in\left[-\sum_{u=i}^{H} r_u,H\right]\subseteq[-H,H],\, \text{implying}\\
&E(\traj{i})\in[-2\tau(s_i) H,2\tau(s_i) H]\subseteq [-2H,2H]\,.
\end{split}\end{align}
Furthermore, since either $\tau(s_i)=0$ or $\norm{\bar{\phi}_\Q(s_i)}_2=1$ (as $\norm{\bar{\phi}_\Q(s_i)}=0$ implies that $\range_\Q(s_i)$ and hence $\tau(s_i)$ are both zero), we have
\begin{align}\label{eq:tr-f-is-e}
\trace(F(\traj{i}))=E(\traj{i})\,,
\end{align}
which was used to establish the last part of \cref{cor:F-components-realizable}.%

\begin{proof}[\textbf{Proof of \cref{lem:v-can-be-projd}}]\label{proof:v-can-be-projd}
We drop the subscripts $(\hat G,\bar\theta)$.
Let $(\hat\vartheta^i_h)_{h\in[2:H],i\in[d_0]}=\hat G\in\bG$.
Let $z=v-w$ be the projection of $v$ to the subspace orthogonal to $Z(\Q,i)$, denoted by ${Z(\Q,i)}^\perp$.
In other words, $z=\Proj_{{Z(\Q,i)}^\perp}v$.
Let $\bM=y^{ki}-\hat F^{ki}$. By the symmetry of $\bM$,
\[
v^\top \bM v = z^\top \bM (v+w) + w^\top \bM w\,.
\]
It is enough to prove therefore that
\[
\frac{\epsilon}{dH^2\omega} \ge z^\top \bM (v+w)\,.
\]
As $\norm{v}_2\le1$ and $\norm{v+w}_2\le2$, and using the definitions and \cref{eq:fact-bounds-E-D}, for any input $(\traj{i})$,
\begin{align*}
\abs{z^\top F(\traj{i}) (v+w)}
&= \abs{\ip{z,\bar\phi_\Q(s_i)}\ip{v+w,\bar\phi_\Q(s_i)} E(\traj{i})}\\
&\le 4H\tau(s_i)\abs{\ip{z,\bar\phi_\Q(s_i)}}
\le 4\range_\Q^{\hat G}(s_i)\frac{\sqrt{2d}H^2}{\epsilon}\abs{\ip{z,\bar\phi_\Q(s_i)}} \\
&\le
4\abs{\ip{z,\bar\phi_\Q(s_i)}}
\max_{a,b,k\in[d_0]} \ip{\phi_\Q(s,a,b),\hat\vartheta_h^k} \frac{\sqrt{2d}H^2}\epsilon\\
&\le 4\norm{\Proj_{{Z(\Q,i)}^\perp}\bar{\phi}_\Q(s)}_2 \norm{\phi'}_2 \max_{k\in[d_0]}\norm{\hat\vartheta_h^k}_2 \frac{\sqrt{2d}H^2}\epsilon\\
&\le 4\norm{\Proj_{{Z(\Q,i)}^\perp} \phi'}_2 \dfour  \frac{\sqrt{2d}H^2}\epsilon\,,
\end{align*}
where $\phi'$ is the unnormalized version of $\bar{\phi}_\Q(s_i)$ of \cref{eq:phiwitness}, that is, $\phi'=\phi_\Q(s_i,a,b)$ for the same $a,b$ as in \cref{eq:phiwitness} (i.e., with the largest $\ell_2$-norm).

As $\Proj_{{Z(\Q,i)}^\perp}\phi'=\Proj_{{Z(\Q,i)}^\perp}(\phi_\Q(s,a)-\phi_\Q(s,b))=\Proj_{{Z(\Q,i)}^\perp} \Q_i (\phi(s,a)-\phi(s,b))$ for some $s\in\cS_i,a,b\in[\cA]$,
and by definition $\Proj_{{Z(\Q,i)}^\perp} \Q_i \mle \precondbound^{-2} I$,
$\norm{\Proj_{{Z(\Q,i)}^\perp}\phi'}_2\le \precondbound^{-2}\norm{\phi(s,a)-\phi(s,b)}_2\le 2 \precondbound^{-2} \featurebound$, so
\begin{align}
\abs{z^\top F(\traj{i}) (v+w)}
&\le 8\precondbound^{-2} \featurebound \dfour  \frac{\sqrt{2d}H^2}\epsilon, \label{eq:Fbound-z}
\end{align}
and hence
\begin{align}
\abs{z^\top \hat F^{ki} (v+w)}
&\le 8\precondbound^{-2} \featurebound \dfour  \frac{\sqrt{2d}H^2}\epsilon \,. \label{eq:hatF-small-z}
\end{align}
To bound $\abs{z^\top y^{ki} (v+w)}$,
note that by the definition $y^{ki}$,
\begin{align*}
z^\top y^{ki} (v+w) &= \frac1n\sum_{j\in[n]}c^j_{ki} \ip{\phi^{mkj}_{\pshort{k}}, \check\theta^{\psuper{mkj}{k},i}} \\
\text{where}\quad\quad
\check\theta^{ti} &= X_{mt}^{-1} \sum_{lkj\in \bI^m(t)} \phi^{lkj}_t \left(z^\top F(S^{lkj}_i,\dots,\Sr^{lkj}) (v+w)\right) \quad\quad\text{for } t\in[i-1]
\end{align*}
Therefore
\[
\abs{z^\top y^{ki} (v+w)} \le \max_{t\in[i-1],s\in\cS_t,a\in[\cA]}  \ip{\phi(s,a), \check\theta^{ti}}\,.
\]
Fix any $t\in[i-1],s\in\cS_t,a\in[\cA]$.
By repeated application of the Cauchy-Schwarz inequality, the fact that $X_{mt} \mge  \lambda I$, the triangle inequality, and using \cref{eq:Fbound-z},
\begin{align*}
\abs{\ip{\phi(s,a), \check\theta^{ti}}}
&\le \norm{\phi(s,a)}_{X_{mt}^{-1}}\norm{\sum_{lkj\in \bI^m(t)} \phi^{lkj}_t \left(z^\top F(S^{lkj}_i,\dots,\Sr^{lkj}) (v+w)\right)}_{X_{mt}^{-1}}\\
&\le \norm{\phi(s,a)}_2 \lambda^{-1/2} \cdot 8\precondbound^{-2}\featurebound \dfour  \frac{\sqrt{2d}H^2}\epsilon \sum_{lkj\in \bI^m(t)}\norm{\phi^{lkj}_t}_{X_{mt}^{-1}}\\
&\le 8\precondbound^{-2}\featurebound^2 \lambda^{-1/2}\dfour \frac{\sqrt{2d}H^2}\epsilon \sqrt{|\bI^m(t)|}\sqrt{\sum_{lkj\in \bI^m(t)}\norm{\phi^{lkj}_t}_{X_{mt}^{-1}}^2}\\
&\le 8\precondbound^{-2}\featurebound^2 \lambda^{-1/2}\dfour \frac{\sqrt{2d}H^2}\epsilon \sqrt{\mmax n H d}\,,
\end{align*}
where we use that $|\bI^m(t)| \le m n H$, $m \le \mmax$ by \cref{lem:m-max}, and that
\[
\sqrt{\sum_{lkj\in \bI^m(t)}\norm{\phi^{lkj}_t}_{X_{mt}^{-1}}^2} = \sqrt{\sum_{lkj\in \bI^m(t)}\trace(X_{mt}^{-1} \phi^{lkj}_t {\phi^{lkj}_t}^\top)}\le \sqrt{\trace{X_{mt}^{-1}X_{mt}}}=\sqrt{d}\,.
\]
Combining with \cref{eq:hatF-small-z}, with an appropriate choice of $\precondbound$, we obtain
\begin{align}\label{eq:precondbound-choice}
\abs{z^\top \bM (v+w)}
&\le 8\precondbound^{-2}\featurebound \dfour \frac{\sqrt{2d}H^2}{\epsilon}\left(1+\featurebound \lambda^{-1/2}\sqrt{\mmax n H d}\right)
\le \frac{\epsilon}{dH^2\omega}
\end{align}
as desired.
\end{proof}

\begin{proof}[\textbf{Proof of \cref{lem:lse-measure-good}}]\label{proof:lse-measure-good}
Choose $\beta$ 
\begin{align}\label{eq:beta-choice}
\beta &\le 2+ 2H\sqrt{2dH(d_0+1) \log\frac{12d_0H\thetabound}{\alpha\xi} + 2\log\frac{\mpmax H^2}{\zeta} + d\log\left(\lambda + \mpmax nH\featurebound^2/d\right)}, \end{align}
satisfying $\beta=\ordot(H^{3/2}d)$ as given in \cref{eq:beta-approx}, and define
\begin{align*}
\xi&=\frac{\epsilon}{5\sqrt{2d}(H+1)^3\featurebound}\left(\min\left\{\epsilon/(dH^2\omega),1/\sqrt{\mpmax nH}\right\}-\eta_0\right).
\end{align*}
Note that subtracting $\eta_0$ keeps $\xi$ positive, and of the same order, by our assumption that $\eta$ is small enough:
$\eta_0\le\frac12 \min\left\{\epsilon/(dH^2\omega),1/\sqrt{\mpmax nH}\right\}$, which follows from \cref{eq:eta-small-assumption}.

We start with a covering argument for the set of functions of the form $v_\pa^\top \bar F_{\hat G \bar\theta}w$, for different choices of $\hat G$, $\bar\theta$, $v$, and $w$.
By \cite[Corollary 4.2.13]{vershynin2018high},
there is a set $C_\xi\subset \B(1)$ with $|C_\xi|\le (3/\xi)^d$ such that
for all $x\in \B(1)$ there exists a $y\in C_\xi$ with $\norm{x-y}_2\le\xi$.
Therefore, there is a set $C_\xi^\times \subset \left(\bigtimes_{h\in[2:H],k\in[d_0]} \B(\dfour )\right)\times \left(\bigtimes_{h\in[2:H]} \B(4d_0H\thetabound/\alpha)\right)\times \B(1)\times \B(1)$
with $|C_\xi^\times|\le (12d_0H\thetabound/(\alpha\xi))^{dH(d_0+1)}$ such that
for any $\hat G=(\hat\vartheta^i_h)_{h\in[2:H],i\in[d_0]}\in\bG$, $\bar\theta\in\bbt$, and $v,w\in\B(1)$,
there exists a $y\in C_\xi^\times$, such that if we let
$\tilde G=(\tilde\vartheta^i)_{h\in[2:H],i\in[d_0]}=(y_{(h-1)d_0+i})_{h\in[2:H],i\in[d_0]}$,
$\tilde\theta=(\tilde\theta_h)_{h\in[2:H]}=(y_{(H-1)d_0+h})_{h\in[2:H]}$,
 and $a=y_{(H-1)(d_0+1)+1}$, $b=y_{(H-1)(d_0+1)+2}$,
then $\tilde G\in\bG$, $\tilde\theta\in\bbt$, $a,b\in\B(1)$,
and
\[
\norm{a-v}_2\le\xi\quad\text{and}\quad \norm{b-w}_2\le\xi\,,\quad\text{and}
\]
\[ %
\norm{\hat\vartheta^i_h-\tilde\vartheta^i}_2\le\xi\,\,\text{and}\,\,
\norm{\bar\theta_h-\tilde\theta_h}_2\le\xi
\,\,\text{for all } h\in[2:H],i\in[d_0]\,.
\]
As a result, for all $s\in\cS\setminus\cS_1$, %
$|\range_\Q^{\hat G}(s) - \range_\Q^{\tilde G}(s)|
\le 2\featurebound \xi$, and
therefore $|\tau_{\hat G \bar\theta}(s) - \tau_{\tilde G \tilde\theta}(s)|
\le 2\sqrt{2d} H\featurebound \xi/\epsilon$. 
Furthermore, $|D_{\hat G \bar\theta}(s,\dots,\sr) - D_{\tilde G \tilde\theta}(s,\dots,\sr)|
\le \featurebound \xi$.
Combining these with the facts that in either case, $\tau(\cdot)\in[0,1]$, $D(\cdot)\in[-H,H]$, and $E^\to(\cdot)\in[-H,H]$ (\cref{eq:fact-bounds-E-D}),
and using the definition of $E$ and $E^\to$,%
we have that for any $i\in[H+1]$ and inputs,
\begin{align*}
|E_{\hat G \bar\theta}(\traj{i})-E_{\tilde G \tilde\theta}(\traj{i})| &\le
4\sqrt{2d}H^2\featurebound\xi/\epsilon + \featurebound \xi +
|E_{\hat G \bar\theta}^\to(\traj{i+1})-E_{\tilde G \tilde\theta}^\to(\traj{i+1})| \\
&= 4\sqrt{2d}H^2\featurebound\xi/\epsilon + \featurebound \xi +
\sum_{j=i+1}^{H}
|E_{\hat G \bar\theta}(\traj{j})-E_{\tilde G \tilde\theta}(\traj{j})| \\
&\le (H+1)5\sqrt{2d}H^2\featurebound\xi/\epsilon\,,
\end{align*}
where the first inequality sums over the contributions of $\tau$, $D$, and $E^\to$, and the second applies induction.
By combining this bound with the bounds on $\norm{v-a}_2$ and $\norm{w-b}_2$, and that $E(\cdot)\in[-2H,2H]$ (\cref{eq:fact-bounds-E-D}) implying that $\bar F(\cdot)\in[-2H,2H]$, for all $s\in\cS\setminus\cS_1$, we have that
\begin{align}\label{eq:cover-barF}
\begin{split}
\abs{v_\pa^\top \bar F_{\hat G \bar\theta}(s)w-a_\pa^\top \bar F_{\tilde G \tilde\theta}b}(s)
& \le
6H\xi+(H+1)5\sqrt{2d}H^2\featurebound\xi/\epsilon \\
&\le 5\sqrt{2d}(H+1)^3\featurebound\xi/\epsilon=\min\{\epsilon/(dH^2\omega),1/\sqrt{\mpmax nH}\}-\eta_0
\end{split}
\end{align}

Take any $m'\in[\mpmax]$ (%
this includes the entire execution of \qpieleanor).
and let the quantities of \cref{sec:lse-targets}
(such as $F$) be calculated with the value of $\Q$ at the beginning iteration $m'$ (Line~\ref{line:opt}).
Take any $t\in[H-1]$, $i\in[t+1:H]$.
Take any $y\in C^\times_\xi$ and assign values to $a,b,\tilde G,\text{and }\tilde\theta$ based on $y$ as above.
For any $lkj\in\bI^m(t)$, observe that given all the history of \qpieleanor interacting with the MDP up to (and including) $S^{lkj}_t,A^{lkj}_t$,
the trajectory $S^{lkj}_{t+1},A^{lkj}_{t+1},\dots,\Sr^{lkj}$ is an independent rollout with policy $\pi^0$,
with its law given by $\P_{\pi^0,S^{lkj}_t,A^{lkj}_t}$.
The random variable $a_\pa^\top F_{\tilde G \tilde\theta}(S^{lkj}_i\dots,\Sr^{lkj})b$ has range $[-2H,2H]$ and expectation (conditioned on this history) $\E_{\pi^0,S^{lkj}_t,A^{lkj}_t} a_\pa^\top \bar F_{\tilde G \tilde\theta}(S_i)b$.
Let $\check\theta_{ti}$ be $\tilde\theta_{ti}$ from \cref{cor:F-components-realizable}, satisfying
$\norm{\check\theta_{ti}}_2\le 1/\sqrt{\lambda}$ and
\cref{eq:barF-realizable} for $a_\pa$, $b$, $\tilde G$, and $\tilde\theta$ instead of $v_\pa$, $w$, $\hat G$, and $\bar\theta$:
\begin{align}\label{eq:checktheta-satisfy}
\E_{\pi^0,s,a} a_\pa^\top\bar F_{\tilde G \tilde\theta}(S_i)b \approx_{\eta_0} \ip{\phi(s,a),\check\theta_{ti}}\,.
\end{align}
Take the sequence $A$ formed of $\phi^{lkj}_t$ (for $lkj\in\bI^m(t)$, in the order that these random variables are observed), and the sequence $X$ formed of
$v_\pa F_{\hat G \bar\theta}(S^{lkj}_i,\dots,\Sr^{lkj}) w$ (for $lkj\in\bI^m(t)$, in the same order),
and the sequence $\Delta$ formed of
$\E_{\pi^0,S^{lkj}_t,A^{lkj}_t} v_\pa \bar F_{\hat G \bar\theta}(S_i) w - \ip{\phi^{lkj}_t,\check\theta_{ti}}$
(for $lkj\in\bI^m(t)$, in the same order, for any $v,w,\hat G,\text{and }\bar\theta$ as in the statement of this lemma).
Then the sequences $A$, $X$, and $\Delta$ satisfy the conditions of \cref{lem:generalised-bandits-book} with a subgaussianity parameter $\sigma=2H$.
Due to this lemma,
with probability at least $1-\zeta/(\mpmax H^2|C^\times_\xi|)$,
for any choice of $v,w,\hat G,\text{and }\bar\theta$ (as above),
\begin{align}\label{eq:lse-guarantee-app}
\norm{\tilde\theta_{ti}-\check\theta_{ti}}_{X_{mt}}&<
\sqrt{\lambda}\norm{\check\theta_{ti}}_2  +  \norm{\Delta}_\infty\sqrt{|\bI^m(t)|} +  2H\sqrt{2\log\left(\frac{\mpmax H^2|C^\times_\xi|}\zeta\right)+\log\left(\frac{\det X_{mt}}{{\lambda}^d}\right)} \\
\text{where}\quad\quad \tilde\theta_{ti}&=
X_{mt}^{-1} \sum_{lkj\in \bI^m(t)} \phi^{lkj}_t v_\pa^\top F_{\hat G \bar\theta}(S^{lkj}_i,\dots,\Sr^{lkj}) w \nonumber
\end{align}
A union bound over all $m'\in[\mpmax]$, $t,i$, and $y\in C^\times_\xi$ guarantees with probability at least $1-\zeta$, the above holds for all choice of these variables, any time
beginning of any iteration (Line~\ref{line:opt}) is executed.
Note that we need the union bound over $m$ because the value of $\Q$ underlying the targets of least-squares estimations can potentially change between iterations.

To finish the proof, under this high-probability event, take any $m,t,i,\hat G, \text{and }\bar\theta$ as in the statement of this lemma, and choose $y\in C^\times_\xi$ as before, to satisfy \cref{eq:cover-barF}.
Combined with \cref{eq:checktheta-satisfy},
this immediately implies that the sequence $\Delta$ formed of quantities with absolute value
\begin{align}\begin{split}\label{eq:cover-corr}
&\abs{\E_{\pi^0,S^{lkj}_t,A^{lkj}_t} v_\pa \bar F_{\hat G \bar\theta}(S_i) w - \ip{\phi^{lkj}_t,\check\theta_{ti}}}\\
&\quad\le \abs{\E_{\pi^0,S^{lkj}_t,A^{lkj}_t} v_\pa \bar F_{\hat G \bar\theta}(S_i) w - a_\pa \bar F_{\tilde G \tilde\theta}(S_i) b}
 + \abs{a_\pa \bar F_{\tilde G \tilde\theta}(S_i) b - \ip{\phi^{lkj}_t,\check\theta_{ti}}}\\
&\quad\le \min\{\epsilon/(dH^2\omega),1/\sqrt{\mpmax nH}\}-\eta_0+ \eta_0
\end{split}\end{align}
satisfies $\norm{\Delta}_\infty\le \min\{\epsilon/(dH^2\omega),1/\sqrt{\mpmax nH}\}$.
Take any $(s,a)\in\cS_t\times[\cA]$, and let $\tilde\theta_{ti}$ and $\check\theta_{ti}$ be as above (in \cref{eq:lse-guarantee-app}) for $v_\pa$, $w$, $\hat G$, and $\bar\theta$.
Note that
\[
v_\pa^\top {\phi(s,a)}^\top \hat\theta^{ti}_{\hat G \bar\theta} w
= \ip{\phi(s,a),\tilde\theta_{ti}}\,,
\]
By the triangle inequality, using Cauchy-Schwarz, and \cref{eq:lse-guarantee-app,eq:cover-corr},
\begin{align}\label{eq:final-lse-beta}\begin{split}
&\abs{v_\pa^\top\left({\phi(s,a)}^\top \hat\theta^{ti}_{\hat G \bar\theta} -
\E_{\pi^0,s,a} \bar F_{\hat G \bar\theta}(S_i) \right)w } \\
 &\quad\le
 \abs{\ip{\phi(s,a),\tilde\theta_{ti}-\check\theta_{ti}}} +
 \abs{\E_{\pi^0,s,a} v_\pa^\top \bar F_{\hat G \bar\theta}(S_i) w - \ip{\ip{\phi(s,a),\check\theta_{ti}}}}\\
 &\quad\le \norm{\phi(s,a)}_{X_{mt}^{-1}}
 \left(\sqrt{\lambda}\norm{\check\theta_{ti}}_2  +  \frac{\sqrt{|\bI^m(t)|}}{\sqrt{\mpmax nH}} +  2H\sqrt{2\log\left(\frac{\mpmax H^2|C^\times_\xi|}\zeta\right)+\log\left(\frac{\det X_{mt}}{{\lambda}^d}\right)} \right)
 + \frac{\epsilon}{dH^2\omega}\\
 &\quad\le \norm{\phi(s,a)}_{X_{mt}^{-1}}
 \left(2+ 2H\sqrt{2dH(d_0+1) \log\frac{12d_0H\thetabound}{\alpha\xi} + 2\log\frac{\mpmax H^2}{\zeta} + d\log\left(\lambda + \mpmax nH\featurebound^2/d\right)}\right) + \frac{\epsilon}{dH^2\omega}\\
 &\quad\le \norm{\phi(s,a)}_{X_{mt}^{-1}}\beta + \frac{\epsilon}{dH^2\omega}\,,
 \end{split}
\end{align}
where in the fourth line we used that $|\bI^m(t)|\le \mpmax nH$, $|C^\times_\xi|\le(12d_0H\thetabound/(\alpha\xi))^{dH(d_0+1)}$, and we used the inequality of arithmetic and geometric means to bound $\det X_{mt}
\le \left(\frac1d \trace X_{mt}\right)^d\le \left(\frac{\trace \lambda I + |\bI^m(t)|\featurebound^2}{d}\right)^d$. %
\end{proof}

\begin{proof}[\textbf{Proof of \cref{lem:avg-measure-good}}]\label{proof:avg-measure-good}
Choose $n$ to satisfy
\begin{align}\label{eq:n-def}
n &= \ceil{64\frac{(dH^2\omega)^2}{\epsilon^2}H^2  \left(2d\log\frac{18dH^3}{\epsilon}+\log\frac{2\mpmax H^2}{\zeta}\right)}\,.
\end{align}
This leads to $n=\ordot(d^5 H^6 / \epsilon^2 )$.

Similarly to the proof of \cref{lem:lse-measure-good}, we start with a covering argument.
This time, as $\hat G$ and $\bar\theta$ are fixed, we only consider $v$ and $w$, to cover $v_\pa^\top\bar F^{(j)}_{t'} w$ and $v_\pa\hat F^{(j)}_{t'} w$.
Let $\xi'=\frac{\epsilon}{12dH^3}$.
There is a set $C_{\xi'}^+\subset \B(1)\times\B(1)$ with $|C_{\xi'}|\le(3/\xi')^{2d}$ such that for all $v,w\in\B(1)$,
there exists an $(a,b)\in C_{\xi'}^+$ with $\norm{v-a}_2\le\xi'$ (and therefore $\norm{v_\pa-a_\pa}_2\le\xi'$), and $\norm{w-b}_2\le\xi'$.
Take such a choice of $(a,b)$ for any $(v,w)$.
As $E(\cdot)\in[-2H,2H]$ by \cref{eq:fact-bounds-E-D},
and $\norm{\bar{\phi}_\Q(\cdot)}_2\le1$,
For $i\in[2:H]$ and any input,
\[
\abs{v_\pa^\top F(\traj{i})w - a_\pa^\top F(\traj{i})b}
\le 6H\xi' = \frac{\epsilon}{2dH^2} \,,
\]
and therefore for any $s\in\cS\setminus\cS_1$, $\abs{v_\pa^\top \bar F(s)w - a_\pa^\top \bar F(s)b}
\le \epsilon/(2dH^2)$.
For $j\in[n]$ let
\[
\tilde F^{ki}_j= \E_{\pi^0,S^{mkj}_{\pshort{k}},A^{mkj}_{\pshort{k}}}  F_{\hat G \bar\theta}(S^{mkj}_i,\dots,\Sr^{mkj}) =  \E_{\pi^0,S^{mkj}_{\pshort{k}},A^{mkj}_{\pshort{k}}} \bar F_{\hat G \bar\theta}(S^{mkj}_i)
\]
By the triangle inequality, for any $k\in[H-1]$, $i\in[k+1:H]$,
\begin{align}\label{eq:hoeff-tofinish}
\begin{split}
&\abs{
v_\pa^\top \left(y^{ki}_{\hat G \bar\theta}-\hat F^{ki}_{\hat G \bar\theta}\right) w} \\
&\quad\le
\abs{\frac1n\sum_{j\in[n]}c^j_{ki} v^\top\left({\phi^{mkj}_{\pshort{k}}}^\top \hat\theta^{\psuper{mkj}{k},i}_{\hat G \bar\theta} -
\tilde F^{ki}_j
\right) w}
+ \abs{\frac1n\sum_{j\in[n]}c^j_{ki} v^\top\left(
\tilde F^{ki}_j
 - F_{\hat G \bar\theta}(S^{mkj}_i,\dots,\Sr^{mkj})\right) w} \\
&\quad\le
\frac1n\sum_{j\in[n]}c^j_{ki}
	\norm{\phi^{mkj}_{\pshort{k}}}_{X_{m,\psuper{mkj}{k}}^{-1}}\beta + \frac{\epsilon}{dH^2\omega}
	+ \frac{\epsilon}{dH^2\omega}
	+ \abs{\frac1n\sum_{j\in[n]}c^j_{ki}a^\top\left(
\tilde F^{ki}_j
 - F_{\hat G \bar\theta}(S^{mkj}_i,\dots,\Sr^{mkj})\right) b}
\,,
\end{split}\end{align}
where the second inequality uses \cref{lem:lse-measure-good} and applies the triangle inequality twice again.
Observe that for all $j\in[n]$, given all the history of \qpieleanor interacting with the MDP up to (and including) $S^{mkj}_{\pshort{k}},A^{mkj}_{\pshort{k}}$ (which also includes the value of $c^j_{ki}$ for $i\in[H+1]$),
the trajectory $S^{mkj}_{\pshort{k}+1},A^{mkj}_{\pshort{k}+1},\dots,\Sr^{mkj}$ is
an independent rollout with policy $\pi^0$,
with its law given by $\P_{\pi^0,S^{mkj}_{\pshort{k}},A^{mkj}_{\pshort{k}}}$.
Therefore, for any fixed $(a,b)\in C^+_{\xi'}$,
$c^j_{ki} a^\top\left(
\tilde F^{ki}_j
 - F_{\hat G \bar\theta}(S^{mkj}_i,\dots,\Sr^{mkj})\right) b$
are independent zero-mean random variables with range $[-4H,4H]$.
Applying Hoeffding's inequality with a union bound over $m',k,i,a$, and $b$, with probability at least $1-\zeta$,
for any of the $m'\in [\mpmax]$ times the beginning of the iteration (Line~\ref{line:opt}) is executed (%
this includes the entire execution of \qpieleanor),
\begin{align*}
\abs{\frac1n\sum_{j\in[n]}c^j_{ki} a^\top\left(
\tilde F^{ki}_j
 - F_{\hat G \bar\theta}(S^{mkj}_i,\dots,\Sr^{mkj})\right) b}
 &\le \frac{8H}{\sqrt{n}} \sqrt{\log\frac{2\mpmax H^2|C^+_{\xi'}|}{\zeta}}\\
 &=\frac{8H}{\sqrt{n}} \sqrt{2d\log\frac{18dH^3}{\epsilon}+\log\frac{2\mpmax H^2}{\zeta}}\le \frac{\epsilon}{dH^2\omega}\,,
\end{align*}
where we used \cref{eq:n-def}.
To finish, note that unless $c^j_{ki}=0$,
$\norm{\phi^{mkj}_{\pshort{k}}}_{X_{m,\psuper{mkj}{k}}^{-1}}< 2(\beta\omega dH)^{-1}$, so we can continue from \cref{eq:hoeff-tofinish}
by bounding the average feature-norm by $\bar\sigma^m_k$ as
\[
\abs{
v_\pa^\top \left(y^{ki}_{\hat G \bar\theta}-\hat F^{ki}_{\hat G \bar\theta}\right) w}
\le \bar\sigma^m_k\beta + 3\frac{\epsilon}{dH^2\omega}\,. \qedhere
\]
\end{proof}

\begin{proof}[\textbf{Proof of \cref{lem:decoy-find}}]
Recall that $(k,i,v)$ are the arguments and $x$ the value of \cref{def:consistency-opt}.
Throughout the proof we write $\Q$ to refer to its value \emph{just before} Line~\ref{line:decoydetect} is executed.
We write $\bullet_\pa$ for $\bullet_{\pa(\Q,i)}$, and
$\bullet_\perp$ for $\bullet_{\perp(\Q,i)}$.
Let $\bM=y^{ki}_{\hat G \bar\theta}-\hat F^{ki}_{\hat G \bar\theta}$.
Therefore, $v^\top \bM v = x > \bar\sigma_k^m\beta\omega+3\frac{\epsilon}{dH^2}$, and by \cref{lem:v-can-be-projd},
$w^\top \bM w > \bar\sigma_k^m\beta\omega+2\frac{\epsilon}{dH^2}$.

Line~\ref{line:decoydetect} changes $\Q_i$ by appending $\Q_i^{-1}w$ to the sequence $C_i$ of vectors from which $\Q$ is calculated according to \cref{eq:q-form}.
\cref{eq:precond-lownorm} lists the conditions on the new sequence $C_i$ that need to be satisfied
for $\Q$ to stay a valid preconditioning.
Consider the third condition, i.e.,
$\norm{\Q_i^{-1}w}_2\le \precondbound$.
Observe that $\Q_i^{-1} \Proj_{Z(\Q,i)} \mle \precondbound^{2} I$ and $\norm{v}_2=1$, therefore
$\norm{\Q_i^{-1}w}_2=\norm{\Q_i^{-1}\Proj_{Z(\Q,i)} v}_2\le \precondbound$.

Now consider the second condition.
To prove that it holds, we need to show that $\norm{\Q_i \Q_i^{-1} w}_2=\norm{w}_2\ge\frac12$.
Let $x=\norm{w}_2^{-1}$. Since $v$ was the argument of the optimization problem, and using \cref{lem:v-can-be-projd},
\[
x^2 w^\top \bM w \le v^\top \bM v \le w^\top \bM w + \frac{\epsilon}{dH^2\omega} \le w^\top \bM w(1+1/2)
\]
Therefore, $\norm{w}_2^2 \ge \frac23$.
We immediately get that
\[
\norm{\Q_i \Q_i^{-1} w}_2^2\ge\frac23\,,
\]
satisfying the second condition.

It remains to prove that the first condition also holds.
First, noting that $\bM$ is symmetric, we can decompose $w^\top \bM w$ as
\[
w^\top \bM w= w_\pa^\top \bM w + w_\pa^\top \bM w_\perp + w_\perp^\top \bM w_\perp\,.
\]
Applying \cref{lem:avg-measure-good} on the first two terms,
\[
w^\top \bM w \le 2\bar\sigma_k^m\beta + 6\frac{\epsilon}{dH^2\omega} + w_\perp^\top \bM w_\perp\,.
\]
Due to $\omega>3$ and $w^\top \bM w > \bar\sigma_k^m\beta\omega+2\frac{\epsilon}{dH^2}$ and the above, $w_\perp\ne \mathbf{0}$.
Let $w'=w_\perp/\norm{w_\perp}_2$.
Since $v$ was the argument of the optimization problem, have that $v^\top \bM v\ge {w'}^\top \bM {w'}$.
Putting this together,
\[
\norm{w_\perp}_2^{-2}w_\perp^\top \bM w_\perp={w'}^\top \bM {w'}\le v^\top \bM v
\le w^\top \bM w + \frac{\epsilon}{dH^2\omega}
 \le 2\bar\sigma_k^m\beta + 7\frac{\epsilon}{dH^2\omega} + w_\perp^\top \bM w_\perp\,,
\]
Since $v^\top \bM v > \bar\sigma_k^m\beta\omega+3\frac{\epsilon}{dH^2}$, $w_\perp^\top \bM w_\perp\ge (\omega-7/3)\left(\bar\sigma_k^m\beta+3\frac{\epsilon}{dH^2\omega}\right)>0$
and therefore dividing the above by $w_\perp^\top \bM w_\perp$,
\begin{align*}
\norm{w_\perp}_2^{-2} &\le \frac{7/3}{\omega-7/3} + 1\\
\norm{w_\perp}_2^2 &\ge \frac1{1+c}\quad\quad\text{for } c=\frac{7/3}{\omega-7/3}\\
\norm{w_\pa}_2^2 &\le 1-\frac{1}{1+c}\quad\quad\text{as } \norm{w}_2\le 1\,.
\end{align*}
Now to prove that the first condition also holds,
\begin{align*}
\sup_{\theta\in\Theta_i}\abs{\ip{\theta,\Q_i^{-1}w}}
&=\sup_{\theta\in\Theta_i^\Q}\abs{\ip{\theta,w}}
\le \sup_{\theta\in\Theta_i^\Q}\norm{\theta}_2\norm{w_\pa}_2 + \sup_{\theta\in\Theta_i^\Q}\abs{\ip{\theta,w_\perp}}\\
&\le \dfour \sqrt{1-\frac{1}{1+c}} + \sup_{\theta\in\Theta_i^\Q}\norm{\theta}_{V(G_h^\Q, \rho_h^\Q)^\dag} \norm{w_\perp}_{V(G_h^\Q, \rho_h^\Q)}\\
&\le \dfour \sqrt{1-\frac{1}{1+c}} + \sqrt{2d w_\perp^\top (\gamma I) w_\perp} \\
&\le \dfour \sqrt{1-\frac{1}{1+c}} + \sqrt{2d\gamma} = \dfour \sqrt{1-\frac{1}{1+c}} + \frac12 \,,
\end{align*}
where in the second line we used \cref{lem:theta-precond-norm-bound} to bound $\sup_{\theta\in\Theta_i^\Q}\norm{\theta}_2$,
and for the second term we used \cref{eq:pseudoinv-cool-for-cauchy} with Cauchy-Schwarz.
In the third line we used \cref{eq:near-opt-2d},
and the definition of $\Proj_\perp$.
Finally in the last line we use that $w_\perp$ is perpendicular to $a_i$ for $i\le d'$ (by definition) and that $\lambda_i\le \gamma$ for $i>d'$.
It is left to prove that $\dfour \sqrt{1-\frac{1}{1+c}}\le\frac12$.
This holds if $c\ge 1/(4\dfoursq -1)$, which is satisfied as $c=1/(3\dfoursq )$, due to $\omega=7\dfoursq +7/3$ (\cref{eq:omega-def}).
\end{proof}

\section{Deferred proofs for \cref{sec:sampcompbound}}

\begin{proof}[\textbf{Proof of \cref{lem:m-max}}]\label{proof:m-max}
The features $\phi^{lkj}_{\pshort{k}}$ are observed by \qpieleanor in the order of increasing $l$, within that increasing $k$, and within that, increasing $j$.
Each time the next $\phi^{lkj}_{\pshort{k}}$ is observed, we
 sum the elliptic potential as follows.

For $i\in[m],r\in[H],u\in[n],t\in[H]$, let
the set of indices observed before $\phi^{iru}_{\pshort{r}}$ whose Phase II (rollout phase) starts at some stage $t$ be:
\[
\bI^{iru}(t) = \{l\in[i],k\in[H], j\in[n]\,:\, lHn+kn+j < iHn+rn+u \text{ and } \psuper{lkj}{k}=t \}
\]
Let a version of this where only the whole iteration $i$'s data is included be
\[
\bJ^{i}(t) = \{l=i,k\in[H], j\in[n]\,:\, \psuper{lkj}{k}=t \}
\]
Let
\[
X_{iru}(t) = \lambda I + \sum_{lkj\in\bI^{iru}(t)} \phi^{lkj}_{\pshort{k}} { \phi^{lkj}_{\pshort{k}}}^\top
\]
Observe that $X_{it}$, defined in \cref{def:opt-problem}, is the version of this that only updates at the start of each iteration $i$, that is,
\[
X_{it} = X_{i11}(t)\,.
\]
The total elliptic potential, observed by the end of iteration $m$ is, writing $k=\psuper{iru}{r}$ on the left hand side:
\[
\sum_{i\in[m],r\in[H],u\in[n]}
\I{k<H+1} \min\left\{1, \norm{\phi^{iru}_k}_{{X_{iru}(k)}^{-1}}^2\right\}
=
\sum_{i\in[m],t\in[H]}\sum_{lkj\in\bJ^i(t)} \min\left\{1, \norm{\phi^{lkj}_t}_{{X_{lkj}(t)}^{-1}}^2\right\} \,.
\]
Applying the elliptical potential lemma (\cref{lem:elliptical-pot}) $H$ times for $t\in[H]$, this can be bounded as
\[
\sum_{t\in[H],i\in[m]}\sum_{lkj\in\bJ^i(t)} \min\left\{1, \norm{\phi^{lkj}_t}_{{X_{lkj}(t)}^{-1}}^2\right\}
 \le 2dH\log\left(1+\frac{Hmn\featurebound^2}{d\lambda}\right)
\]
On the other hand, by \cref{lem:ok-to-update-x-infrequently},
then switching to an $\ell_1$-bound, then observing that by definition, $\sum\bar\sigma^i_k$ sums the same quantities but caps them by some threshold,
\begin{align*}
\sum_{t\in[H],i\in[m]}\sum_{lkj\in\bJ^i(t)} \min\left\{1, \norm{\phi^{lkj}_t}_{{X_{lkj}(t)}^{-1}}^2\right\}
&\ge
\sum_{t\in[H],i\in[m]}
	\min\left\{1, \frac12
		\sum_{lkj\in\bJ^i(t)} \norm{\phi^{lkj}_t}_{X_{it}^{-1}}^2\right\}\\
&\ge
\sum_{i\in[m]}
	\min\left\{1, \frac12 \sum_{t\in[H]}
		\sum_{lkj\in\bJ^i(t)} \norm{\phi^{lkj}_t}_{X_{it}^{-1}}^2\right\}\\
&\ge
\sum_{i\in[m]}
	\min\left\{1, \frac1{2Hn} \left(\sum_{t\in[H]}
		\sum_{lkj\in\bJ^i(t)} \norm{\phi^{lkj}_t}_{X_{it}^{-1}}\right)^2\right\}\\
&\ge
\sum_{i\in[m]}
	\min\left\{1, \frac1{2Hn} \left(
	n\sum_{k\in[H]} \bar\sigma^i_k\right)^2\right\}\\
\end{align*}
Whenever an iteration finishes without returning in Line~\ref{line:return}, %
$\sum_{k\in[H]} \bar\sigma^m_k > \epsilon/(dH^2\beta\omega)$. %
Therefore,
\begin{align*}
2dH\log\left(1+\frac{Hmn\featurebound^2}{d\lambda}\right)
&\ge \sum_{i\in[m]}
	\min\left\{1, \frac1{2Hn} \left(
	n\sum_{k\in[H]} \bar\sigma^i_k\right)^2\right\}\\
&\ge \sum_{i\in[m]}
	\min\left\{1, \frac1{2H}n \left(
	\frac{\epsilon}{dH^2\beta\omega}\right)^2\right\}\\
&\ge \sum_{i\in[m]}
	\min\left\{1, Hd/\beta^2\right\}
	= mHd/\beta^2\,,
\end{align*}
Therefore, even for the iteration that returns in Line~\ref{line:return},
\[
m\le \beta^2 \log\left(1+\frac{Hmn\featurebound^2}{d\lambda}\right)+1=\mmax\,.\qedhere
\]
\end{proof}

\section{Deferred proofs for \cref{sec:perf}}

\begin{proof}[\textbf{Proof of \cref{lem:d-induction}}]\label{proof:d-induction}
For notational simplicity we drop the subscripts $(\hat G,\bar\theta)$.
We first use the usual high-probability bounds on the least squares predictor and Hoeffding's inequality on the empirical mean quantities,
to prove that with probability at least $1-3\zeta$,
during the %
execution of \qpieleanor
whenever Line~\ref{line:cons-check-pass} is executed,
for all $k\in[H]$,
\begin{align}\label{eq:magic-for-induction}
\E_{\pi^{mk},s_1} \tilde c_{k,H+1}
C(S_{p(k)})
\le
\E_{\pi^{mk},s_1} \sum_{u=p(k)}^{H}R_u + \tilde c_{k,H+1}
E^\to(S_{p(k)+1},\dots,\Sr)
+ 2\bar\sigma_k^m\beta\omega dH+4\frac{\epsilon}{H}\,.
\end{align}
The proof of this is presented as \cref{lem:d-induction-subresult}.

Next, to prove the statement for $k\in[H]$, assume by induction that \cref{eq:d-induction} holds for $i\in[k+1:H]$.

Observe that \skippypolicy performs a rollout with policy $\pi^0$ for the rest of the episode starting from stage $p(k)+1$, that is,
$1=A_{p(k)+1}=\dots=A_H$.
Therefore, the law of the random variables $S_{p(k)+1},\dots,\Sr$,
given $(S_{p(k)},A_{p(k)})$ is fully determined by the dynamics of the MDP, and is independent of the values of $p(k+1), \dots,p(H)$.
Therefore, %
\begin{align}\begin{split}\label{eq:middle-complication}
\E_{\pi^{mk},s_1} \tilde c_{k+1,H+1} C(S_{p(k+1)})
&=
\E_{\pi^{mk},s_1} \tilde c_{k+1,H+1} D(S_{p(k+1)},\dots,\Sr) + \sum_{u=p(k+1)}^{H}R_u \\
&=
\E_{\pi^{mk},s_1} \tilde c_{k,H+1} E^\to(S_{p(k)+1},\dots,\Sr) + \sum_{u=p(k+1)}^{H} R_u \,,
\end{split}\end{align}
where we use \cref{eq:fact:E-to-probab-form},
and that $\pi^{mk}$ (\skippypolicy)
is in phase II after stage $p(k)$, but
defines the the mapping $p(\cdot)$ independently of whether the policy is in phase I or phase II, in such a way that for any $H\ge j > p(k)$,
\[
\P_{\pi^{mk},s_1}\left[p(k+1)=j \,\big|\, p(k), S_{p(k)},A_{p(k)}\right]
=
\P_{\pi^{mk},s_1}\left[\tau(S_j)\prod_{j'=p(k)+1}^{j-1}(1-\tau(S_{j'})) \,\big|\, p(k), S_{p(k)},A_{p(k)}\right]
\,.
\]
Combining \cref{eq:middle-complication} with \cref{eq:magic-for-induction},
\begin{align*}
\E_{\pi^{mk},s_1} \tilde c_{k,H+1}
C(S_{p(k)})
&\le
 \E_{\pi^{mk},s_1}
 \sum_{u=p(k)}^{p(k+1)-1} R_u
+
\tilde c_{k+1,H+1} C(S_{p(k+1)})
 + 2\bar\sigma_k^m\beta\omega dH+4\frac{\epsilon}{H}\,.
\end{align*}
By \cref{rem:pol-same},
$
\E_{\pi^{mk},s_1} \tilde c_{k+1,H+1} C(S_{p(k+1)})
=
\E_{\pi^{m,k+1},s_1} \tilde c_{k+1,H+1} C(S_{p(k+1)}) = \bar C^{k+1}\,.
$
Therefore, combining with the inductive hypothesis,
\begin{align*}
\E_{\pi^{mk},s_1} \tilde c_{k,H+1}
C(S_{p(k)})
&\le
\E_{\pi^{mk},s_1}
\sum_{u=p(k)}^{p(k+1)-1} R_u
+
\bar C^{k+1}
+ 2\bar\sigma_k^m\beta\omega dH+4\frac{\epsilon}{H}\\
&\le
\E_{\pi^{mk},s_1}
\sum_{u=p(k)}^{p(k+1)-1} R_u +
\E_{\pi^{mH},s_1} \sum_{u=p(k+1)}^{H} R_u
+ 2\sum_{i=k}^H\bar\sigma_k^m\beta\omega dH+4(H-k+1)\frac{\epsilon}{H}\\
&=
\E_{\pi^{mH},s_1} \sum_{u=p(k)}^{H} R_u
+ 2\sum_{i=k}^H\bar\sigma_k^m\beta\omega dH+4(H-k+1)\frac{\epsilon}{H}
\,
\end{align*}
where the last equation uses \cref{rem:pol-same} again,
finishing the induction.
\end{proof}

\begin{lemma}\label{lem:d-induction-subresult}
Adopt the notation of \cref{lem:d-induction}.
With probability at least $1-3\zeta$,
during the execution of \qpieleanor, %
whenever Line~\ref{line:cons-check-pass} is executed,
for all $k\in[H]$,
\[
\E_{\pi^{mk},s_1} \tilde c_{k,H+1}
C(S_{p(k)})
\le
\E_{\pi^{mk},s_1}
\sum_{u=p(k)}^{H} R_u +
\tilde c_{k,H+1}
E^\to(S_{p(k)+1},\dots,\Sr)
+ 2\bar\sigma_k^m\beta\omega dH+4\frac{\epsilon}{H}\,.
\]
\end{lemma}
\begin{proof}
We refer as $\hat\theta$ to the value of the argument of \cref{def:opt-problem} recorded in Line~\ref{line:opt}.
For $k\in[H]$, recall the definition of $\bar\sigma^m_k$ (\cref{eq:barsigma-def}), along with the fact that unless $c^j_{k,H+1}=0$,
$\norm{\phi^{mkj}_{\pshort{k}}}_{X_{m,\psuper{mkj}{k}}^{-1}}< 2(\beta\omega dH)^{-1}$, we get a useful bound on the average norm of the features under consideration:
\begin{align}\label{eq:avg-uncertainty-bound}
\frac1n\sum_{j\in[n]}c^j_{k,H+1}\norm{ \phi^{mkj}_{\pshort{k}}}_{X_{m,\psuper{mkj}{k}}^{-1}} &\le \bar\sigma^m_k\,.
\end{align}
If Line~\ref{line:cons-check-pass} is executed, the consistency check passed, and therefore for all $k\in[H-1], i\in[k+1:H]$,
\begin{align}\label{eq:checkpass-consequence}
\trace\left(y^{ki} - \hat F^{ki}\right) \le \bar\sigma_k^m\beta\omega d+3\frac{\epsilon}{H^2}
\end{align}
For $t\in[H]$ let the least-squares predictor of rewards sums under the policy $\pi^0$ be
\[
\check\theta^{t,H+1} = X_{mt}^{-1} \sum_{lkj\in \bI^m(t)} \phi^{lkj}_t  \sum_{u=t}^H R^{lkj}_u \,.
\]
For $k\in[H]$
and $j\in[n]$ let us introduce the shorthand
\[
R^{mkj}_{k\to}=\sum_{u=\psuper{mkj}{k}}^H R^{mkj}_u\,,
\]
and similarly when the trajectory is clear from context: $R_{k\to}=\sum_{u=p(k)}^H R_u$.
 For $k\in[H]$ let %
\begin{align*}
\hat E^k &= \frac1n\sum_{j\in[n]}c^j_{k,H+1} \left(E^\to(S^{mkj}_{\pshort{k}+1},\dots,\Sr^{mkj}) + R^{mkj}_{k\to}\right) \\
\hat C^k &=\frac1n\sum_{j\in[n]}c^j_{k,H+1}
C(S^{mkj}_{\pshort{k}}) \\
y^{k,H+1} &= \frac1n\sum_{j\in[n]}c^j_{k,H+1}\ip{\phi^{mkj}_{\pshort{k}}, \check\theta^{\psuper{mkj}{k},H+1}}\\
z^{k,H+1} &= \frac1n\sum_{j\in[n]}c^j_{k,H+1} R^{mkj}_{k\to}
\end{align*}

For $t\in[H-1]$, $i\in[t+1:H]$, along with $\check\theta^{t,H+1}$, let
\[
\check\theta^{ti} = X_{mt}^{-1} \sum_{lkj\in \bI^m(t)} \phi^{lkj}_t \trace( F(S^{lkj}_i,\dots,\Sr^{lkj}))
=X_{mt}^{-1} \sum_{lkj\in \bI^m(t)} \phi^{lkj}_t E(S^{lkj}_i,\dots,\Sr^{lkj})\,,
\]
where the second equality is by \cref{eq:tr-f-is-e}.
Observe that for any $v\in\R^d$, $\trace(v^\top \hat\theta^{ti}) = \ip{v,\check\theta^{ti}}$.
Therefore, for $k\in[H]$,
\[
y^{k,H+1} + \sum_{i=k+1}^H \trace(y^{ki})
=\frac1n\sum_{j\in[n]}\sum_{i=k+1}^{H+1}c^j_{ki}\ip{ \phi^{mkj}_{\pshort{k}}, \check\theta^{\psuper{mkj}{k},i}}
=\frac1n\sum_{j\in[n]}c^j_{k,H+1}\ip{ \phi^{mkj}_{\pshort{k}}, \sum_{i=\psuper{mkj}{k}+1}^{H+1}\check\theta^{\psuper{mkj}{k},i}}
\]
For any $t\in[H]$, by the definitions,
\begin{align*}
\sum_{i=t+1}^{H+1}\check\theta^{ti} & =
X_{mt}^{-1} \sum_{lkj\in \bI^m(t)} \phi^{lkj}_t \left(\sum_{i=t+1}^H E(S^{lkj}_i,\dots,\Sr^{lkj}) + \sum_{u=t}^H R^{mkj}_{u}\right) \\
&= X_{mt}^{-1} \sum_{lkj\in \bI^m(t)} \phi^{lkj}_t \left(E^\to(S^{lkj}_{t+1},\dots,\Sr^{lkj}) + \sum_{u=t}^H R^{mkj}_{u}\right)=\hat\theta_t
\end{align*}
Plugging this into the previous calculation,
\begin{align}\label{eq:fact-eleanor-nonopt-pred}
\begin{split}
y^{k,H+1} + \sum_{i=k+1}^H \trace(y^{ki})
&=\frac1n\sum_{j\in[n]}c^j_{k,H+1}\ip{ \phi^{mkj}_{\pshort{k}}, \hat\theta_{\psuper{mkj}{k}}} \\
&\ge \frac1n\sum_{j\in[n]}c^j_{k,H+1}\ip{ \phi^{mkj}_{\pshort{k}}, \bar\theta_{\psuper{mkj}{k}}} \\
&\quad\quad-  \frac1n\sum_{j\in[n]}c^j_{k,H+1}\norm{ \phi^{mkj}_{\pshort{k}}}_{X_{m,\psuper{mkj}{k}}^{-1}} \norm{\bar\theta_{\psuper{mkj}{k}}-\hat\theta_{\psuper{mkj}{k}}}_{X_{m,\psuper{mkj}{k}}} \\
&\ge \frac1n\sum_{j\in[n]}c^j_{k,H+1}\ip{ \phi^{mkj}_{\pshort{k}}, \bar\theta_{\psuper{mkj}{k}}} - \bar\sigma^m_k \optnormconst \\
&\ge \frac1n\sum_{j\in[n]}c^j_{k,H+1}\clip_{[0,H]}\ip{ \phi^{mkj}_{\pshort{k}}, \bar\theta_{\psuper{mkj}{k}}} - \bar\sigma^m_k \optnormconst\\
&=  \frac1n\sum_{j\in[n]}c^j_{k,H+1}C(S^{mkj}_{\pshort{k}}) - \bar\sigma^m_k \optnormconst
= \hat C^k - \bar\sigma^m_k \optnormconst \,,
\end{split}\end{align}
where the first inequality uses Cauchy-Schwarz. The second inequality bounds the average of the first norm by \cref{eq:avg-uncertainty-bound}, and
the bound on the second norm (for any $j$) is by definition of \cref{def:opt-problem}.
The third inequality relies on the fact that $c^j_{k,H+1}=0$ if the clipped inner product is negative, and the final equality is
due to the definition of $C$ along with the fact that $A^{mkj}_{\pshort{k}}=\pi^+(S^{mkj}_{\pshort{k}})$, as this is the last state in the trajectory where \skippypolicy takes the inner-product maximizing action ($\pi^+$) before rolling out with $\pi^0$.

By \cref{eq:tr-f-is-e,eq:hatF-def}, %
we have that
\begin{align}\label{eq:fact-y-sum}
\begin{split}
z^{k,H+1} + \sum_{i\in[k+1:H]} \trace(\hat F^{ki})
&=
\frac1n\sum_{j\in[n]}c^j_{k,H+1} \left(\sum_{i=\psuper{mkj}{k}+1}^{H+1}E(S^{mkj}_i,\dots,\Sr^{mkj}) + R^{mkj}_{k\to}\right)\\
&=\frac1n\sum_{j\in[n]}c^j_{k,H+1} \left(E^\to(S^{mkj}_{\pshort{k}+1},\dots,\Sr^{mkj}) + R^{mkj}_{k\to} \right)
=\hat E^k
\end{split}\end{align}
Combining \cref{eq:fact-eleanor-nonopt-pred,eq:fact-y-sum},
\begin{align}\label{eq:bary-hatE-diff}
\begin{split}
\hat C^k-\hat E^k &\le \bar\sigma^m_k \optnormconst
+ \left(y^{k,H+1}-z^{k,H+1}\right) + \sum_{i\in[k+1:H]}\trace(y^{ki} - \hat F^{ki})\\
&\le \bar\sigma^m_k \optnormconst + \left(\abs{\E_{\pi^{mk},s_1} c_{k,H+1} R_{k\to} - z^{k,H+1}} + \abs{y^{k,H+1} - \E_{\pi^{mk},s_1} c_{k,H+1} R_{k\to}}\right) + \bar\sigma_k^m\beta\omega dH+3\frac{\epsilon}{H}
\end{split}\end{align}
where the sum (last term) is bounded by \cref{eq:checkpass-consequence}, and we apply a triangle inequality on the second term.
To continue bounding this term,
we apply Hoeffding's inequality on the independent random variables $c^j_{k,H+1} R_{k\to}$ (for $j\in[n]$) that have range $[0,H]$, along with a union bound over the iteration $m'\in[\mpmax]$ and $k\in[H]$, to get that with probability at least $1-\zeta$,
\begin{align}\label{eq:hoeff-term1}
\abs{\E_{\pi^{mk},s_1} c_{k,H+1} R_{k\to} - z^{k,H+1}} \le \frac{H}{\sqrt{n}}\sqrt{\log\frac{2\mpmax H}\zeta} \le \frac{\epsilon}{dH^2\omega}\,.
\end{align}
The remaining term $\abs{y^{k,H+1} - \E_{\pi^{mk},s_1} c_{k,H+1} R_{k\to}}$ is bounded using the realizability of $q^{\pi^0}$ (\cref{def:q-pi-realizable}) as follows.
Take any $t\in[H]$.
By definition %
there exists $\theta_t^\star\in\Theta^\Q_t\subseteq\B(\thetabound)$, such that for all $s\in\cS_t$ and $a\in[\cA]$,
$q^{\pi^0}(s,a) \approx_{\eta} \ip{\phi(s,a),\theta_t^\star}$.
Take the sequence $A$ formed of $\phi^{lkj}_t$ (for $lkj\in\bI^m(t)$, in the order that these random variables are observed), and the sequence $X$ formed of
$R^{mkj}_{k\to}$ (for $lkj\in\bI^m(t)$, in the same order),
and the sequence $\Delta$ formed of
$q^{\pi^0}(S^{lkj}_t,A^{lkj}_t) - \ip{\phi^{lkj}_t, \theta_t^\star}$
(for $lkj\in\bI^m(t)$, in the same order).
Then the sequences $A$, $X$, and $\Delta$ satisfy the conditions of \cref{lem:generalised-bandits-book} with a subgaussianity parameter $\sigma=H$.
Due to this lemma, applied with a union bound over $m'\in[\mpmax]$ and $t\in[H]$,
with probability at least $1-\zeta$,
\begin{align*}%
\norm{\check\theta^{t,H+1}-\theta_t^\star}_{X_{mt}}&<
\sqrt{\lambda}\norm{\theta_t^\star}_2  +  \norm{\Delta}_\infty\sqrt{|\bI^m(t)|} +  H\sqrt{2\log\left(\frac{\mpmax H}\zeta\right)+\log\left(\frac{\det X_{mt}}{{\lambda}^d}\right)} \\
&\le 2 + H\sqrt{2\log\frac{\mpmax H}\zeta+\log\left(\frac{\det X_{mt}}{{\lambda}^d}\right)}\le \beta\,,
\end{align*}
by \cref{eq:final-lse-beta}.
Therefore by Cauchy-Schwarz and \cref{eq:avg-uncertainty-bound},
\begin{align*}
\abs{y^{k,H+1} - \E_{\pi^{mk},s_1} c_{k,H+1} R_{k\to}}
&\le \frac1n\sum_{j\in[n]}c^j_{k,H+1}\left(\norm{ \phi^{mkj}_{\pshort{k}}}_{X_{m,\psuper{mkj}{k}}^{-1}} \norm{\check\theta^{\psuper{mkj}{k},H+1}-\theta_{\psuper{mkj}{k}}^\star}_{X_{mt}} + \eta\right) \\
&\le
\bar\sigma^m_k\beta + \eta\le \bar\sigma^m_k\beta +\frac{\epsilon}{dH^2\omega}\,.
\end{align*}
Combining this with \cref{eq:bary-hatE-diff,eq:hoeff-term1},
\begin{align}\label{eq:before-interrupt-c}
\hat C^k-\hat E^k
&\le 1.5\bar\sigma_k^m\beta\omega dH+3\frac{\epsilon}{H}+2\frac{\epsilon}{dH^2\omega}\,.
\end{align}

We introduce the following notation for $j\in[n]$, $k\in[H+1]$, $i\in[H+1]$:
\begin{align*}%
\bar c^j_{ki}&=\I{\psuper{mkj}{k}<i \text{ and } \norm{\phi^{mkj}_{\pshort{k}}}_{X_{m,\psuper{mkj}{k}}^{-1}} \ge 2(\beta\omega dH)^{-1} \text{ and } \ip{\phi^{mkj}_{\pshort{k}},\bar\theta_{\psuper{mkj}{k}}}\ge0} \\
\hat c^j_{ki}&=\I{\psuper{mkj}{k}<i \text{ and } \ip{\phi^{mkj}_{\pshort{k}},\bar\theta_{\psuper{mkj}{k}}}<0}\,,
\end{align*}
such that for all $j$,
\begin{align}\label{eq:c-versions-sum}
\tilde c^j_{ki}=c^j_{ki}+\bar c^j_{ki}+\hat c^j_{ki}\,.
\end{align}
Continuing from \cref{eq:before-interrupt-c}, as $E^\to(\traj{i})+\sum_{u=i}^{H} r_u\ge 0$ by \cref{eq:fact-bounds-E-D}, and if $\hat c^j_{k,H+1}=1$ then $C(S^{mkj}_{\pshort{k}})=0$, we have that
\[
\frac1n\sum_{j\in[n]} (c^j_{k,H+1}+\hat c^j_{k,H+1})
\left(C(S^{mkj}_{\pshort{k}}) - \left(E^\to(S^{mkj}_{\pshort{k}+1},\dots,\Sr^{mkj}) + R^{mkj}_{k\to}\right)\right)
\le 1.5\bar\sigma_k^m\beta\omega dH+3\frac{\epsilon}{H}+2\frac{\epsilon}{dH^2\omega}\,.
\]
As (even if $\bar c^j_{k,H+1}=1$) $C(S^{mkj}_{\pshort{k}})\le H$,
\[
\frac1n\sum_{j\in[n]} \bar c^j_{k,H+1} C(S_{p(k)})
\le H \bar\sigma^m_k/(2(\beta\omega dH)^{-1}) = \frac12\bar\sigma_k^m\beta\omega dH \,,
\]
which combined with the previous inequality and \cref{eq:c-versions-sum} yields
\[
\frac1n\sum_{j\in[n]} \tilde c^j_{k,H+1}
\left(C(S^{mkj}_{\pshort{k}}) - \left(E^\to(S^{mkj}_{\pshort{k}+1},\dots,\Sr^{mkj})+R^{mkj}_{k\to}\right)\right)
\le 2\bar\sigma_k^m\beta\omega dH+3\frac{\epsilon}{H}+2\frac{\epsilon}{dH^2\omega}\,.
\]

Observe that the random variables $\tilde c^j_{k,H+1}\left(C(S^{mkj}_{\pshort{k}}) - \left(E^\to(S^{mkj}_{\pshort{k}+1},\dots,\Sr^{mkj})+R^{mkj}_{k\to}\right)\right)$ are independent (for $j\in[n]$) with range $[-2H,H]$ (\cref{eq:fact-bounds-E-D}).
By Hoeffding's inequality, with probability at least $1-\zeta$, for all iteration $m'\in[\mpmax]$ (%
this includes the entire execution of \qpieleanor) and $k\in[H]$,
\begin{align*}
&\Bigg|\E_{\pi^{mk},s_1} \tilde c_{k,H+1}\left(C(S_{p(k)}) - \left(E^\to(S_{p(k)+1},\dots,\Sr)+R_{k\to}\right)\right)\\
&\quad\quad
- \frac1n\sum_{j\in[n]} \tilde c^j_{k,H+1}
\left(C(S^{mkj}_{\pshort{k}}) - \left(E^\to(S^{mkj}_{\pshort{k}+1},\dots,\Sr^{mkj})+R^{mkj}_{k\to}\right)\right)
\Bigg|\\
&\quad\quad\quad\quad\le \frac{4H}{\sqrt{n}} \sqrt{\log\frac{2\mpmax H}{\zeta}} \le \frac{\epsilon}{dH^2\omega}\,.
\end{align*}
Combining with the previous bound, %
under the intersection of the high-probability events referred to above, which by a union bound has a probability of at least $1-3\zeta$, we have that for all $k\in[H]$,
\[
\E_{\pi^{mk},s_1} \tilde c_{k,H+1}
C(S_{p(k)})
\le
\E_{\pi^{mk},s_1}
\sum_{u=p(k)}^H R_u +
\tilde c_{k,H+1}
E^\to(S_{p(k)+1},\dots,\Sr)
+ 2\bar\sigma_k^m\beta\omega dH+4\frac{\epsilon}{H}\,. \qedhere
\]
\end{proof}

\section{Deferred proofs for \cref{sec:optimism}}

\begin{proof}[\textbf{Proof of \cref{lem:optimism}}]\label{proof:optimism}
Let $m$ be the current iteration.
Unlike in previous lemmas, here we introduce $(\hat G,\bar\theta)$ that does \emph{not} refer to the outcome of \cref{def:opt-problem}.
Instead, let
$\hat G=(\vartheta_h^i)_{h\in[2:H],i\in[d_0]}\in\bG$ be the correct guess.
For $h=H,\dots,1$, $\bar\theta_h$ is defined in sequence along with the behavior of a policy $\pi$ on stage $h$.

For $h=H,\dots,1$,
assuming that this process already defined $\bar\theta_{h+1},\dots,\bar\theta_H$ (in \cref{eq:opt-bartheta}),
let $\pi$ be the policy that,
for any $t>h$ and $s\in\cS_t$,
takes action on $s$ as
$\pi^+_{\hat G \bar\theta}(s)$ with probability $\tau_{\hat G \bar\theta}(s)$, and action $1$ with probability $1-\tau_{\hat G \bar\theta}(s)$ ($\tau$ is defined in \cref{eq:tau-pi+-def}).
Simultaneously, using the second part of \cref{cor:F-components-realizable}, define ${\tilde\theta}_{hi}\in\B(4d_0\thetabound/\alpha)$ for $i\in[h+1:H]$
to satisfy for all $s\in\cS_h$, $a\in[\cA]$:
\[
\E_{\pi^0,s,a}\trace(\bar F_{\hat G \bar\theta}(S_i)) \approx_{\eta_0} \ip{\phi(s,a),{\tilde\theta}_{hi}}\,.
\]
We also define ${\tilde\theta}_{h,H+1}\in\B(\thetabound)$
to satisfy for all $s\in\cS_h$, $a\in[\cA]$:
\[
\E_{\pi^0,s,a} \sum_{u=h}^H R_u \approx_{\eta} \ip{\phi(s,a),{\tilde\theta}_{h,H+1}}\,.
\]
By \cref{eq:tr-f-is-e},
\begin{align}\label{eq:eto-optimism}
\E_{\pi^0,s,a} \sum_{i\in[h+1:H]} \trace(\bar F_{\hat G \bar\theta}(S_i)) +\sum_{u=h}^H R_u
 = \E_{\pi^0,s,a} E_{\hat G \bar\theta}^\to(S_{h+1},\dots,\Sr) +\sum_{u=h}^H R_u \approx_{H\eta_0} \ip{\phi(s,a),\bar\theta_h}
\,,
\end{align}
where we define
\begin{align}\label{eq:opt-bartheta}
\bar\theta_h&=\sum_{i\in[h+1:H+1]}{\tilde\theta}_{hi}\,.
\end{align}

We first show that $(\hat G,\bar\theta)$ is feasible for \cref{def:opt-problem}.
Clearly, $\norm{\bar\theta_h}_2\le 4d_0H\thetabound/\alpha$.
For any $i\in[h+1:H]$, let
\[
\hat\theta_{hi}= X_{mh}^{-1}\sum_{lkj\in\bI^m(h)} \phi^{lkj}_h \trace(F_{\hat G \bar\theta}(S^{lkj}_{h+1},\dots,\Sr^{lkj}))\,,
\]
and let
\[
\hat\theta_{h,H+1}= X_{mh}^{-1}\sum_{lkj\in\bI^m(h)} \phi^{lkj}_h \sum_{u=h}^H R^{lkj}_u\,.
\]

Then, $\hat\theta$ of \cref{def:opt-problem}
satisfies for all $h\in[H]$, by \cref{eq:tr-f-is-e},
\[
\hat\theta_h=\sum_{i\in[h+1:H+1]} \hat\theta_{hi}\,.
\]
To show that $(\hat G,\bar\theta)$ is feasible, it thus suffices to
show for all $h\in[H]$, $i\in[h+1:H+1]$, that
$\norm{{\tilde\theta}_{hi}-\hat\theta_{hi}}_{X_{mh}}\le \beta$.

Fix any $h\in[H]$ and $i\in[h+1:H+1]$.
Take the sequence $A$ formed of $\phi^{lkj}_t$ (for $lkj\in\bI^m(h)$, in the order that these random variables are observed).
For $i<H+1$ take the sequence $X$ formed of
$\trace( F_{\hat G \bar\theta}(S^{lkj}_i,\dots,\Sr^{lkj}))$ (for $lkj\in\bI^m(h)$, in the same order),
and the sequence $\Delta$ formed of
$\E_{\pi^0,S^{lkj}_h,A^{lkj}_h} \trace( \bar F_{\hat G \bar\theta}(S_i)) - \ip{\phi^{lkj}_h,\tilde\theta_{hi}}$
(for $lkj\in\bI^m(h)$, in the same order).
For $i=H+1$, the sequence $X$ is formed of $\sum_{u=h}^H R^{lkj}_u$, and $\Delta$ is formed of
$q^{\pi^0}(S^{lkj}_h,A^{lkj}_h) - \ip{\phi^{lkj}_h,\tilde\theta_{hi}}$.
Then the sequences $A$, $X$, and $\Delta$ satisfy the conditions of \cref{lem:generalised-bandits-book} with a subgaussianity parameter $\sigma=H$.
Due to this lemma, applied with a union bound over $m'\in[\mpmax]$, $t$, and $i$,
with probability at least $1-\zeta$,
\begin{align*}%
\norm{\hat\theta_{hi}-\tilde\theta_{hi}}_{X_{mh}}&<
\sqrt{\lambda}\norm{\tilde\theta_{hi}}_2  +  \norm{\Delta}_\infty\sqrt{|\bI^m(t)|} +  H\sqrt{2\log\left(\frac{\mpmax H^2}\zeta\right)+\log\left(\frac{\det X_{mt}}{{\lambda}^d}\right)} \\
&\le 2 + H\sqrt{2\log\frac{\mpmax H^2}\zeta+\log\left(\frac{\det X_{mt}}{{\lambda}^d}\right)}\le \beta\,,
\end{align*}
by \cref{eq:final-lse-beta}.

Next, we show that the resulting policy $\pi$ is near-optimal.
Assume by induction on $h=H,\dots,1$, that for all
$t\in[h+1:H]$, all $s\in\cS_t$ and $a\in[\cA]$,
\begin{align}
v^{\pi}(s) &\ge v^\star(s)-(H-t+1)(\epsilon/H+2H^2\eta_0) \quad\text{and}\quad\label{eq:opt-ind1}\\
\ip{\phi(s,a),\bar\theta_t}
&\approx_{(H-t+1)H\eta_0}
q^{\pi}(s,a)\,.\label{eq:opt-ind2}
\end{align}
To prove the above for $t=h$ as well, take any $s\in\cS_h,a\in[\cA]$.
Introduce the random variable $P$
that, for a trajectory following $\P_{\pi^0,s,a}$,
takes as its value the index of the first Bernoulli draw of $1$ (starting from index $h+1$),
when the Bernoullis have means $\tau_{\hat G \bar\theta}(S_j)$ for $j\in[h+1:H]$, and takes the value $H+1$ if all of these Bernoullis have outcome $0$. %
Write $\E_{\pi^0,s,a,P}[\cdot]$ for $\E_{\pi^0,s,a}\E_P[\cdot\,|\,S_{h+1},\dots,\Sr]$.
Then,
\begin{align*}
\E_{\pi^0,s,a} E_{\hat G \bar\theta}^\to(S_{h+1},\dots,\Sr) +\sum_{u=h}^H R_u
&=\E_{\pi^0,s,a,P} D_{\hat G \bar\theta}(S_{P},\dots,\Sr) +\sum_{u=h}^H R_u \\
&= \E_{\pi^0,s,a,P} \sum_{u=h}^{P-1} R_u + \I{P<H+1} C_{\hat G \bar\theta}(S_P)
\end{align*}
where we use \cref{eq:fact:E-to-probab-form}.
Combining with \cref{eq:eto-optimism},
\begin{align*}
\ip{\phi(s,a),\bar\theta_h} &\approx_{H\eta_0}
\E_{\pi^0,s,a,P} \sum_{u=h}^{P-1} R_u + \I{P<H+1} C_{\hat G \bar\theta}(S_{P})\\
&=\E_{\pi^0,s,a,P} \sum_{u=h}^{P-1} R_u + \I{P<H+1}\clip_{[0,H]} \ip{\phi(S_P,\pi^+_{\hat G \bar\theta}(S_P)),\bar\theta_P} \\
&\approx_{(H-h)H\eta_0}
\E_{\pi^0,s,a,P} \sum_{u=h}^{P-1} R_u + \I{P<H+1} q^\pi(S_P,\pi^+_{\hat G \bar\theta}(S_P))\,,
\end{align*}
where we used the inductive assumption along with the fact that action-values are bounded in $[0,H]$.
Observe also that
\[
q^\pi(s,a)= \E_{\pi^0,s,a,P} \sum_{u=h}^{P-1} R_u + \I{P<H+1} q^\pi(S_P,\pi^+_{\hat G \bar\theta}(S_P))\,,
\]
and therefore
\[
\ip{\phi(s,a),\bar\theta_h} \approx_{(H-h+1)H\eta_0} q^\pi(s,a)\,,
\]
proving \cref{eq:opt-ind2} of the inductive assumption for $t=h$.

To show \cref{eq:opt-ind1} for $t=h$, by \cref{eq:opt-ind2} for $t=h$ and the inductive assumption for $t>h$,
\[
\ip{\phi(s,a),\bar\theta_h} \approx_{H^2\eta_0}
q^\pi(s,a) \ge q^\star(s,a)-(H-h)(\epsilon/H+2H^2\eta_0) \,.\]
Either $\pi$ chooses the action $a'$ maximizing the inner product above, for which
\[
q^\pi(s,a')\ge \max_{a\in[\cA]} q^\star(s,a)-(H-h)(\epsilon/H+2H^2\eta_0) - 2H^2\eta_0 \ge v^\star(s) -(H-h+1)(\epsilon/H+2H^2\eta_0) \,,
\]
or it chooses action $1$.
This can only happen with non-zero probability if
$\tau_{\hat G \bar\theta}(s)<1$, in which case
we have by definition that $\range_\Q^{\hat G}(s)=\range_\Q(s)\le \frac \epsilon{\sqrt{2d}H}$.
Combining with \cref{eq:vstar-vs-range} and \cref{prop:opt-design-gap-and-vstar-gap},
$\range(s)\le \frac\epsilon H$, and therefore, using \cref{eq:opt-ind1} for $t=h+1$, in this case
\begin{align*}
q^\pi(s,1)&\ge q^\star(s,1) -(H-h)(\epsilon/H+2H^2\eta_0) \\
&\ge v^\star(s)- \frac\epsilon H-2\eta-(H-h)(\epsilon/H+2H^2\eta_0) \ge v^\star(s)- (H-h+1)(\epsilon/H+2H^2\eta_0)\,.
\end{align*}
Therefore for any choice of action $a'$ of policy $\pi$ in state $s$,
$q^\pi(s,a')\ge v^\star(s)-(H-h+1)(\epsilon/H+2H^2\eta_0)$.
Therefore
\[
v^\pi(s)\ge v^\star(s)-(H-h+1)(\epsilon/H+2H^2\eta_0)\,,
\]
finishing the induction.

We thus conclude that
\[
v^{\pi}(s_1)\ge v^\star(s_1)-\epsilon - 2H^3\eta_0\,.
\]
Combined with \cref{eq:opt-ind2} of the inductive assumption, the value of \cref{def:opt-problem} can be bounded as
\[
C_{\hat G \bar\theta}(s_1) =
\clip_{[0,H]} \ip{\phi(s_1,\pi(s_1)),\bar\theta_1}
\ge H^2\eta_0 + v^\pi(s_1)
\ge v^\star(s_1)-2\epsilon\,,
\]
by assumption on $\eta$ being relatively small (\cref{eq:eta-small-assumption}).
\end{proof}

\section{Deferred lemmas}

\begin{lemma}[Elliptical potential, Lemma 19.4 from \cite{LaSze19:book}]\label{lem:elliptical-pot}
Let $V_0\in\R^{d\times d}$ be positive definite and $a_1\dots,a_n\in\R^d$ be a sequence of vectors with $\norm{a_t}_2\le L <\infty$ for all $t\in[n]$, $V_t=V_0+\sum_{s\le t} a_s a_s^\top$. Then,
\[
\sum_{t=1}^n \min\left\{1, \norm{a_t}_{V_{t-1}^{-1}}^2\right\}
\le 2\log\left(\frac{\det V_n}{\det V_0}\right)
\le 2d\log\left(\frac{\trace V_0 + n L^2}{d\det(V_0)^{1/d}}\right)\,.
\]
\end{lemma}

\begin{lemma}\label{lem:ok-to-update-x-infrequently}
Let $V\in\R^{d\times d}$ be a symmetric positive definite matrix and $(a_i)_{i\in [n]}$ be a sequence of $n$ $d$-dimensional real vectors.
Let $V_i=V+\sum_{j\in[i]} a_j a_j^\top$. Then,
\[
\sum_{i\in[n]} \norm{a_i}_{V_i^{-1}}^2 \ge \min\left\{ 1,  \frac12 \sum_{i\in[n]} \norm{a_i}_{V^{-1}}^2 \right\}
\]
\end{lemma}
\begin{proof}
If $\sum_{i\in[n]} a_i a_i^\top \mle V$, then $V_i \mle 2V$, and therefore
\[
\sum_{i\in[n]} \norm{a_i}_{V_i^{-1}}^2 \ge \sum_{i\in[n]} \norm{a_i}_{2V^{-1}}^2 = \frac12 \norm{a_i}_{V^{-1}} \,.
\]
Otherwise,
$\sum_{i\in[n]} a_i a_i^\top V^{-1}$ has an eigenvalue that is at least $1$. As all the other eigenvalues are non-negative (as $V$ is symmetric positive definite), we have that
\[
\sum_{i\in[n]}\norm{a_i}_{V^{-1}}^2 = \Tr\left(\sum_{i\in[n]} a_i a_i^\top V^{-1} \right) \ge 1 \,. \qedhere
\]
\end{proof}

\section{Estimation error blow-up guarantees}

We borrow \cref{ass:pre-bandits-thm,thm:bandits-book-20.5} from \cite{LaSze19:book} and refer the reader to the book for the corresponding proof.

\begin{assumption}[Prerequisites for Theorem~\ref{thm:bandits-book-20.5}]\label{ass:pre-bandits-thm}
Let $\lambda>0$.
For $k\in\N^+$, let $A_k$ be random variables taking values in $\R^d$.
For some $\theta_\star\in\R^d$, let $X_k=\ip{A_k, \theta_\star}+\eta_k$ for all $k\in\N^+$.
Here, $\eta_k$ is a conditionally 1-subgaussian random variable (``noise''), ie. it satisfies:
\begin{align*}%
\text{for all } \alpha\in\R \text{ and } t\ge 1,\,\quad\quad \E[\exp(\alpha \eta_k) \,|\, \F_{k-1}] \le \exp\left(\frac{\alpha^2}{2}\right) \quad\text{a.s.},\,
\end{align*}
where $\F_{k-1}$ is such that $A_1, X_1,\ldots,A_{k-1},X_{k-1},A_k$ are $\F_{k-1}$-measurable.
\end{assumption}
\begin{theorem}[\cite{LaSze19:book}, Theorem 20.5]\label{thm:bandits-book-20.5}
Let $\zeta\in(0,1)$.
Under Assumption~\ref{ass:pre-bandits-thm}, with probability at least $1-\zeta$, it holds that for all $k\in\N$,
\[
\norm{\hat\theta_k-\theta_\star}_{V_k(\lambda)} < \sqrt{\lambda}\norm{\theta_\star}_2+\sqrt{2\log\left(\frac1\zeta\right)+\log\left(\frac{\det V_k(\lambda)}{\lambda^d}\right)}\,,
\]
where for $k\in\N$,
\begin{align*}
V_k(\lambda)&=\lambda I + \sum_{s=1}^k A_s A_s^\top \\
\hat\theta_k&=V_k(\lambda)^{-1}\sum_{s=1}^k X_s A_s
\end{align*}
\end{theorem}

We generalize this theorem to handle non-zero-mean noise with parametrized subgaussianity.
To handle non-zero-mean noise, we use \cite[Lemma 8]{zanette2020learning}. We state the lemma here and refer the reader to \cite{zanette2020learning} for the proof:
\begin{lemma}[\cite{zanette2020learning}, Lemma 8]
\label{lem:eleanor-lse-misspec}
For $n\in N^+$, let $\{A_i\}_{i=1,\dots,n}$ be any sequence of vectors in $\R^d$ and
$\{\Delta_i\}_{i=1,\dots,n}$ be any sequence of scalars such that $|\Delta_i|\le \xi\in\R$ with $\xi\ge0$. For any $\lambda\ge0$ and $V(\lambda)=\sum_{i=1}^n A_i A_i^\top + \lambda I$ we have:
\[
\norm{\sum_{i=1}^n A_i\Delta_i}^2_{V(\lambda)^{-1}}\le n\xi^2
\]
\end{lemma}

\begin{lemma}\label{lem:generalised-bandits-book}
Let $\zeta\in(0,1)$, $\lambda>0$, $\sigma>0$, and $\xi\ge 0$.
For $k\in\N^+$, let $A_k$ be random variables taking values in $\R^d$.
For some $\theta_\star\in\R^d$, let $\tilde X_k=\ip{A_k, \theta_\star}+\eta_k$ for all $k\in\N^+$.
Here, $\eta_k$ is a conditionally $\sigma$-subgaussian random variable, ie. it satisfies:
\[
\text{for all } \alpha\in\R \text{ and } t\ge 1,\,\quad\quad \E[\exp(\alpha \eta_k) \,|\, \F_{k-1}] \le \exp\left(\frac{\alpha^2\sigma^2}{2}\right) \quad\text{a.s.},\,
\]
where $\F_{k-1}$ is such that $A_1, \tilde X_1,\ldots,A_{k-1},\tilde X_{k-1},A_k$ are $\F_{k-1}$-measurable.
With probability at least $1-\zeta$, it holds that for
any sequence $\{\Delta_i\}_{i=1,\dots}$ such that $|\Delta_i|\le\xi$,
for all $k\in\N$,
\[
\norm{\hat\theta_k-\theta_\star}_{V_k(\lambda)} < \sqrt{\lambda}\norm{\theta_\star}_2  +  \xi\sqrt{k} +  \sigma\sqrt{2\log\left(\frac1\zeta\right)+\log\left(\frac{\det V_k(\lambda)}{{\lambda}^d}\right)} \,.
\]
where for $k\in\N$,
\begin{align*}
X_k&=\tilde X_k+\Delta_k\\
V_k(\lambda)&=\lambda I + \sum_{s=1}^k A_s A_s^\top \\
\hat\theta_k&=V_k(\lambda)^{-1}\sum_{s=1}^k X_s A_s
\end{align*}
\end{lemma}
\begin{proof}
Let $X'_k=(X_k-\Delta_k)/\sigma_k$, $A'_k=A_k/\sigma_k$, $\lambda'=\lambda/\sigma_k^2$, and $\theta'_\star=\theta_\star$,
$V'_k(\lambda')=\lambda' I + \sum_{s=1}^k A'_s {A'_s}^\top$, and
$\hat\theta'_k=V'_k(\lambda')^{-1}\sum_{s=1}^k X'_s A'_s$.
By assumption, %
$X'_k$, $A'_k$, $\lambda'$ and $\theta'_\star$ then satisfy Assumption~\ref{ass:pre-bandits-thm}. Therefore by applying Theorem~\ref{thm:bandits-book-20.5}, with probability at least $1-\zeta$, it holds that for all $k\in\N$,
\[
\norm{\hat\theta'_k-\theta_\star}_{V'_k(\lambda')} < \sqrt{\lambda'}\norm{\theta_\star}_2+\sqrt{2\log\left(\frac1\zeta\right)+\log\left(\frac{\det V'_k(\lambda')}{{\lambda'}^d}\right)}\,.
\]
Under this high-probability event, since $V'_k(\lambda')=V_k(\lambda)/\sigma^2$, %
substituting into the previous display yields
\begin{align}\label{eq:lse-bound1}
\norm{\hat\theta'_k-\theta_\star}_{V_k(\lambda)} < \sqrt{\lambda}\norm{\theta_\star}_2+\sigma\sqrt{2\log\left(\frac1\zeta\right)+\log\left(\frac{\det V_k(\lambda)}{{\lambda}^d}\right)}\,.
\end{align}
Take any sequence $\{\Delta_i\}_{i=1,\dots}$ such that $|\Delta_i|\le\xi$ and apply the triangle inequality:
\begin{align}\label{eq:triangle-lse-errors}
\norm{\hat\theta_k-\theta_\star}_{V_k(\lambda)}
\le
\norm{\hat\theta'_k-\theta_\star}_{V_k(\lambda)}
+ \norm{\hat\theta'_k-\hat\theta_k}_{V_k(\lambda)}\,,
\end{align}
so it remains to bound $\norm{\hat\theta'_k-\hat\theta_k}_{V_k(\lambda)}$.
\begin{align}\begin{split}\label{eq:lse-bound-2}
\norm{\hat\theta'_k-\hat\theta_k}_{V_k(\lambda)}
&=\norm{V'_k(\lambda')^{-1}\sum_{s=1}^k X'_s A'_s-V_k(\lambda)^{-1}\sum_{s=1}^k X_s A_s}_{V_k(\lambda)}\\
&=\norm{V_k(\lambda)^{-1}\sum_{s=1}^k (X_s-\Delta_s) A_s-V_k(\lambda)^{-1}\sum_{s=1}^k X_s A_s}_{V_k(\lambda)}\\
&=\norm{V_k(\lambda)^{-1}\sum_{s=1}^k \Delta_s A_s}_{V_k(\lambda)}=\norm{\sum_{s=1}^k \Delta_s A_s}_{V_k(\lambda)^{-1}}\\
&\le \sqrt{k}\xi\,,
\end{split}\end{align}
where the final inequality uses \cref{lem:eleanor-lse-misspec}. The proof is finished by plugging in the bounds of \cref{eq:lse-bound1,eq:lse-bound-2} into the triangle inequality of \cref{eq:triangle-lse-errors}.
\end{proof}

\end{document}

%% file: preamble.tex
\usepackage[utf8]{inputenc}
\usepackage[english]{babel}
\usepackage{latexsym}
\usepackage{mathrsfs}
\usepackage{mathtools}
\usepackage{bm}
\usepackage{datetime}
\usepackage{tikz}
\usepackage{listings}
\usepackage{mdframed}
\usepackage{pgfplots}
\usepackage{pgfplotstable}
\usepackage{dsfont}
\usepackage{color}
\usepackage{colortbl}
\usepackage{pifont}
\usepackage[bf,font=small,singlelinecheck=off]{caption}

\usepackage[textsize=tiny,disable]{todonotes}
\usepackage{microtype} %
\usepackage{float}
\usepackage{xfrac} %
\usepackage{xspace}

\renewcommand{\phi}{\varphi}

\usepackage{nicefrac}
\usepackage{wrapfig}

\usepackage[normalem]{ulem}

\newcommand{\Q}{Q}
\renewcommand{\P}{\mathop{\mathcal{P}}}

\newcommand{\F}{\mathcal{F}}

\DeclareMathOperator{\Dists}{\mathcal{M}_1}

\DeclareMathOperator{\poly}{poly}
\DeclareMathOperator{\polylog}{polylog}
\definecolor{emerald}{rgb}{0.31, 0.78, 0.47}

\usepackage{enumerate}
\newtheorem{theorem}{Theorem}[section]
\newtheorem{proposition}[theorem]{Proposition}
\newtheorem{problem}[theorem]{Problem}
\newtheorem{optproblem}[theorem]{Optimization Problem}
 \newtheorem{lemma}[theorem]{Lemma}
 \newtheorem{corollary}[theorem]{Corollary}
 \newtheorem{definition}[theorem]{Definition}
\newtheorem{assumption}[theorem]{Assumption}
 \newtheorem{remark}[theorem]{Remark}

\usepackage{aliascnt}

\newcommand{\abs}[1]{\left|#1\right|}

\newcommand{\E}{\mathop{\mathbb{E}}}

\newcommand{\ip}[1]{\left\langle #1 \right\rangle}

\newcommand{\norm}[1]{\left\|#1\right\|}
\newcommand{\R}{\mathbb{R}}
\newcommand{\N}{\mathbb{N}}
\newcommand{\cA}{\mathcal{A}}

\newcommand{\cP}{\mathcal{P}}

\newcommand{\cR}{\mathcal{R}}
\newcommand{\cS}{\mathcal{S}}

\newcommand{\one}[1]{\mathbb{I}\{#1\}}

\renewcommand{\epsilon}{\varepsilon}
\newcommand{\ceil}[1]{\left\lceil {#1} \right\rceil}

\DeclareMathOperator*{\argmin}{arg\ min}
\DeclareMathOperator*{\argmax}{arg\ max}

\newif\ifsup\suptrue

\usepackage{algorithm}
\usepackage{algpseudocode}

%% file: qpi.bbl
\begin{thebibliography}{12}
\providecommand{\natexlab}[1]{#1}
\providecommand{\url}[1]{\texttt{#1}}
\expandafter\ifx\csname urlstyle\endcsname\relax
  \providecommand{\doi}[1]{doi: #1}\else
  \providecommand{\doi}{doi: \begingroup \urlstyle{rm}\Url}\fi

\bibitem[Du et~al.(2019)Du, Kakade, Wang, and Yang]{Du_Kakade_Wang_Yan_2019}
Simon~S Du, Sham~M Kakade, Ruosong Wang, and Lin~F Yang.
\newblock Is a good representation sufficient for sample efficient
  reinforcement learning?
\newblock In \emph{International Conference on Learning Representations}, 2019.

\bibitem[Jin et~al.(2020{\natexlab{a}})Jin, Yang, Wang, and
  Jordan]{Jin_Yang_Wang_Jordan_2019}
Chi Jin, Zhuoran Yang, Zhaoran Wang, and Michael~I Jordan.
\newblock Provably efficient reinforcement learning with linear function
  approximation.
\newblock In \emph{Conference on Learning Theory}, pages 2137--2143,
  2020{\natexlab{a}}.

\bibitem[Jin et~al.(2020{\natexlab{b}})Jin, Yang, Wang, and
  Jordan]{jin2020provably}
Chi Jin, Zhuoran Yang, Zhaoran Wang, and Michael~I Jordan.
\newblock Provably efficient reinforcement learning with linear function
  approximation.
\newblock In \emph{Conference on Learning Theory}, pages 2137--2143. PMLR,
  2020{\natexlab{b}}.

\bibitem[Jin et~al.(2021)Jin, Yang, and Wang]{jin2021pessimism}
Ying Jin, Zhuoran Yang, and Zhaoran Wang.
\newblock Is pessimism provably efficient for offline rl?
\newblock In \emph{International Conference on Machine Learning}, pages
  5084--5096. PMLR, 2021.

\bibitem[Lattimore and Szepesv{\'a}ri(2020)]{LaSze19:book}
T.~Lattimore and Cs. Szepesv{\'a}ri.
\newblock \emph{Bandit Algorithms}.
\newblock Cambridge University Press, 2020.

\bibitem[Lattimore et~al.(2020)Lattimore, Szepesv\'ari, and Weisz]{LaSzeGe19}
Tor Lattimore, {\relax Cs}aba Szepesv\'ari, and Gell{\'e}rt Weisz.
\newblock Learning with good feature representations in bandits and in {RL}
  with a generative model.
\newblock In \emph{ICML}, pages 9464--9472, 2020.

\bibitem[Todd(2016)]{todd2016minimum}
Michael~J Todd.
\newblock \emph{Minimum-volume ellipsoids: Theory and algorithms}.
\newblock SIAM, 2016.

\bibitem[Vershynin(2018)]{vershynin2018high}
Roman Vershynin.
\newblock \emph{High-dimensional probability: An introduction with applications
  in data science}, volume~47.
\newblock Cambridge university press, 2018.

\bibitem[Wagenmaker et~al.(2022)Wagenmaker, Chen, Simchowitz, Du, and
  Jamieson]{wagenmaker2022reward}
Andrew Wagenmaker, Yifang Chen, Max Simchowitz, Simon~S Du, and Kevin Jamieson.
\newblock Reward-free {RL} is no harder than reward-aware {RL} in linear
  {M}arkov decision processes.
\newblock \emph{arXiv preprint arXiv:2201.11206}, 2022.

\bibitem[Weisz et~al.(2022)Weisz, Gy{\"o}rgy, Kozuno, and
  Szepesvari]{weisz2022confident}
Gell{\'e}rt Weisz, Andr{\'a}s Gy{\"o}rgy, Tadashi Kozuno, and Csaba Szepesvari.
\newblock Confident approximate policy iteration for efficient local planning
  in $q^\pi$-realizable {MDP}s.
\newblock In \emph{Advances in Neural Information Processing Systems}, 2022.

\bibitem[Yin et~al.(2022)Yin, Hao, Abbasi-Yadkori, Lazi{\'c}, and
  Szepesv{\'a}ri]{yin2022efficient}
Dong Yin, Botao Hao, Yasin Abbasi-Yadkori, Nevena Lazi{\'c}, and Csaba
  Szepesv{\'a}ri.
\newblock Efficient local planning with linear function approximation.
\newblock In \emph{International Conference on Algorithmic Learning Theory},
  pages 1165--1192. PMLR, 2022.

\bibitem[Zanette et~al.(2020)Zanette, Lazaric, Kochenderfer, and
  Brunskill]{zanette2020learning}
Andrea Zanette, Alessandro Lazaric, Mykel Kochenderfer, and Emma Brunskill.
\newblock Learning near optimal policies with low inherent {B}ellman error.
\newblock In \emph{International Conference on Machine Learning}, pages
  10978--10989. PMLR, 2020.

\end{thebibliography}
